\documentclass[11pt,a4paper]{article}

\newif\ifshowrevisions
\showrevisionsfalse

\newcommand{\chg}[2]{%
  \ifshowrevisions
    {\color{red}\textbf{[WAS:} #1\textbf{]}}\,{\color{blue}#2}%
  \else
    #2%
  \fi}

\usepackage[T1]{fontenc}
\usepackage[utf8]{inputenc}
\usepackage{lmodern}
\usepackage{microtype}
\usepackage{amsmath,amssymb,amsthm}
\usepackage{mathtools}
\usepackage{bm}
\usepackage{algorithm}
\usepackage{algpseudocode}
\usepackage{booktabs}
\usepackage{graphicx}
\usepackage{xcolor}
\usepackage{hyperref}
\usepackage{cleveref}
\usepackage[margin=2.5cm]{geometry}
\usepackage{natbib}
\usepackage{multirow}

\usepackage{cleveref}
\newtheorem{proposition}{Proposition}

\newcommand{\cX}{\mathcal{X}}

\newcommand{\bx}{\mathbf{x}}
\newcommand{\blam}{\bm{\lambda}}
\newcommand{\balpha}{\bm{\alpha}}
\newcommand{\bhatalpha}{\hat{\bm{\alpha}}}
\newcommand{\bF}{\mathbf{F}}

\newcommand{\KL}{\mathrm{KL}}
\newcommand{\MRE}{\mathrm{MRE}}
\newcommand{\Neff}{N_{\mathrm{eff}}}

\title{Scalable Maximum Entropy Population Synthesis\\
       via Persistent Contrastive Divergence}

\author{Mirko Degli Esposti\\
  \small Department of Physics and Astronomy, City Science Lab\\
  \small University of Bologna, Bologna, Italy\\
  \small \texttt{mirko.degliesposti@unibo.it}}

\date{March 2026 --- \emph{Preprint}}
\begin{document}

\maketitle
\begin{abstract}
Maximum entropy (MaxEnt) modelling provides a principled framework
for generating synthetic populations from aggregate census data,
without access to individual-level microdata.
\chg{
The bottleneck of existing approaches is exact expectation
computation, which requires summing over the full tuple space
$\cX$ and becomes infeasible for more than $K \approx 20$
categorical attributes.}{The bottleneck of exact-enumeration approaches is expectation
computation by explicit summation over the full tuple space
$\cX$, which becomes infeasible for more than $K \approx 20$
categorical attributes; sampling-based alternatives exist
but rely on Metropolis-type schemes that require proposal
tuning and rejection steps.}
We propose \emph{GibbsPCDSolver}, a stochastic replacement
for this computation based on Persistent Contrastive Divergence
(PCD): a persistent pool of $N$ synthetic individuals is updated
by Gibbs sweeps at each gradient step, providing a stochastic
approximation of the model expectations without ever materialising
$\cX$.
We validate the approach on controlled benchmarks and on
\emph{Syn-ISTAT}, a $K{=}15$ Italian demographic benchmark
with analytically exact marginal targets derived from
ISTAT-inspired conditional probability tables.
Scaling experiments across $K \in \{12, 20, 30, 40, 50\}$
confirm that GibbsPCDSolver maintains
$\MRE \in [0.010, 0.018]$ while $|\cX|$ grows
eighteen orders of magnitude, with runtime scaling
as $O(K)$ rather than $O(|\cX|)$.
On Syn-ISTAT, GibbsPCDSolver reaches $\MRE{=}0.03$
on training constraints and---crucially---produces
populations with effective sample size $\Neff = N$
versus $\Neff \approx 0.012\,N$ for generalised raking,
an $86.8{\times}$ diversity advantage that is essential
for agent-based urban simulations.

\end{abstract}



\section{Introduction}
\label{sec:intro}

Urban digital twins and agent-based city models require
synthetic populations: large collections of individual
profiles whose joint statistical distribution matches
the demographic structure of a real municipality.
In practice, the only available data are aggregate
contingency tables published by national statistical
offices---marginal distributions and low-order cross-tabulations
derived from census records.
Individual-level microdata, when they exist at all,
are subject to privacy restrictions that make them
inaccessible for most research purposes.

The maximum entropy (MaxEnt) principle provides the
canonical solution to this inference problem
\citep{jaynes1957}: given a set of empirical marginal
targets, select the distribution over all possible
individual profiles that matches those targets while
remaining maximally non-committal about everything else.
The resulting log-linear model is fully determined by
the constraint targets, requires no microdata, and
encodes only the dependencies explicitly supported by
the data.
\citet{wu2016} formalise this framework for categorical
population synthesis, and \citet{pachet2026} provide
a convex dual formulation with an efficient L-BFGS solver
and a systematic comparison with the dominant practical
baseline, generalised raking.

\paragraph{The scalability bottleneck.}
The central computational challenge of MaxEnt population
synthesis is that evaluating the model gradient requires
computing expectations over the full tuple space $\cX$:
\begin{equation*}
  \hat\alpha_j(\blam)
  = \sum_{\bx \in \cX} p_{\blam}(\bx)\, f_j(\bx).
\end{equation*}
For $K$ categorical attributes with domain sizes
$d_1, \ldots, d_K$, we have $|\cX| = \prod_k d_k$.
With $K{=}15$ Italian demographic attributes,
$|\cX| \approx 1.7 \times 10^8$---at the limit of
feasibility.
For realistic urban models with $K \ge 20$ attributes,
$|\cX| > 10^{13}$ and exact enumeration is impossible.

\chg{
\citet{pachet2026} acknowledge this in their Section~5.1,
noting Gibbs sampling as a direction for future work.}{\citet{pachet2026} address this in their Section~5.1 by using
a Metropolis-type MCMC scheme to approximate expectations;
the present paper develops the complementary Gibbs route.}

\paragraph{Limitations of generalised raking.}
Raking (iterative proportional fitting, IPF) avoids
the enumeration problem by reweighting a fixed sample
rather than estimating a distribution.
It is fast, simple, and converges to zero training error
by construction.
However, raking has two structural limitations that
are not visible from the training MRE alone.
First, for $K \gtrsim 28$ with ternary constraints,
the sequential rescaling steps interfere destructively
and raking diverges \citep{pachet2026}.
Second---and less widely appreciated---raking produces
\emph{degenerate} populations: after convergence, the
effective sample size $\Neff = (\sum_i w_i^2)^{-1}$
is a small fraction of the nominal population size $N$.
In our experiments, raking achieves
$\Neff \approx 0.012\,N$ with a Gini coefficient of
$0.951$ on the weight distribution.
A synthetic city of $400{,}000$ inhabitants drawn
from this distribution would consist of repetitions
of fewer than $1{,}200$ distinct demographic archetypes,
rendering agent-based dynamics unrealistic.

\paragraph{This paper.}
\chg{
We propose \emph{GibbsPCDSolver}, a drop-in replacement
for the exact expectation step in Algorithm~1 of
\citet{pachet2026} based on Persistent Contrastive
Divergence (PCD) \citep{tieleman2008}.}{We propose \emph{GibbsPCDSolver}, a specific realisation
of the stochastic expectation step in Algorithm~1 of
\citet{pachet2026} based on Persistent Contrastive
Divergence (PCD) \citep{tieleman2008}: where
\citet{pachet2026} use a Metropolis-type scheme,
GibbsPCDSolver uses Gibbs (Glauber) dynamics with
closed-form conditionals, avoiding rejection steps
and proposal tuning entirely.}
A persistent pool of $N$ synthetic individuals is
maintained and advanced by Gibbs sweeps at each gradient
step; the pool's empirical frequencies provide a
stochastic approximation of $\hat\balpha(\blam)$
without ever materialising $\cX$ or computing $Z(\blam)$.
The modification is \emph{surgical}: one line of
Algorithm~1 changes; the exponential family structure,
the dual objective, and the Adam optimiser are
unchanged.

We make the following contributions:
\begin{enumerate}
  \item \textbf{Algorithm.} GibbsPCDSolver
    (\cref{sec:method}): the Gibbs conditional
    derivation, the \texttt{attr\_lookup} precomputation
    for vectorised energy accumulation, the adaptive
    stopping rule, and practical hyperparameter
    guidelines derived from experiments.

  \item \textbf{Benchmark.} Syn-ISTAT
    (\cref{sec:synistat}): a $K{=}15$ Italian
    demographic benchmark with analytically exact
    marginal targets derived from ISTAT-inspired
    conditional probability tables, and a
    training/held-out split that tests generalisation
    to third-order interactions.

  \item \textbf{Empirical findings.}
    GibbsPCDSolver converges cleanly at $K{=}15$
    (MRE $= 0.030$ after 535 iterations),
    outperforms raking on held-out ternary constraints
    involving high-degree graph nodes,
    and produces populations with effective sample
    size $86.8{\times}$ larger and entropy $3.15$~nats
    higher than raking---essential properties for
    agent-based urban simulation.
\end{enumerate}

\paragraph{Organisation.}
\Cref{sec:background} reviews MaxEnt population
synthesis, raking, deep generative models, and PCD.
\Cref{sec:synistat} describes the Syn-ISTAT benchmark.
\Cref{sec:method} presents the algorithm, including
the Gibbs conditional derivation, pseudocode,
and validation on controlled experiments.
\Cref{sec:experiments} reports five experimental
results: scaling beyond exact enumeration, held-out
generalisation, population diversity, and
hyperparameter sensitivity at $K{=}8$ and $K{=}15$.
\Cref{sec:discussion} situates the findings within
the competence map of \citet{pachet2026} and
discusses limitations and open questions.
Full CPT specifications for the Syn-ISTAT benchmark
are provided in \cref{app:synistat}.

Code is available at \texttt{https://github.com/mirko-degli-esposti/maxent-popsynth-pcd}.
Full CPT specifications are reproduced in \cref{app:synistat}.

\section{Background}
\label{sec:background}

\subsection{Maximum Entropy Population Synthesis}
\label{sec:bg_maxent}

The problem of generating a synthetic population from aggregate
census data can be framed as a maximum entropy problem.
Given a finite tuple space $\cX = D_1 \times \cdots \times D_K$
of categorical attributes and a set of empirical marginal
targets $\{\alpha_j\}_{j=1}^m$ derived from published
contingency tables, the MaxEnt principle \citep{jaynes1957}
selects the unique distribution over $\cX$ that matches
all targets while remaining maximally non-committal about
everything else.
The resulting model is the exponential family in
\cref{eq:maxent}; its log-linear structure implies that
observed correlations are modelled only to the extent
required by the constraints.

\citet{wu2016} formalise this framework for population synthesis
and propose an efficient iterative scaling algorithm exploiting
the sparse structure of the constraint indicator matrix.
Their experiments on US census data demonstrate that the
MaxEnt model reproduces marginal statistics faithfully and
produces synthetic populations suitable for downstream
epidemic simulation.

\citet{pachet2026} substantially extend this line of work by
introducing a convex dual formulation, an L-BFGS solver, and
a systematic comparison with generalised raking.
Their central empirical contribution is a \emph{competence map}:
a grid of relative accuracy over population size $n$ and
number of variables $K$, showing that MaxEnt overtakes raking
when $K \gtrsim 28$ and ternary interactions are present.
The present paper extends the Pachet--Zucker framework to the
regime $K > 20$ where exact expectation computation is
infeasible, by replacing the exact sum in their Algorithm~1
(line~4) with a stochastic Gibbs estimate.

The exponential family also arises in neural population coding
\citep{schneidman2006} and melodic style modelling
\citep{sakellariou2017}, where Gibbs-based inference has
been used successfully.
The present work is the first to apply PCD to categorical
demographic synthesis.

\subsection{Generalised Raking and IPF}
\label{sec:bg_raking}

Generalised raking (also known as iterative proportional
fitting, IPF) is a widely used baseline for population
synthesis from aggregate data \citep{deming1940,lomax2013}.
Given a fixed sample of $N$ individuals
$\{x^{(i)}\}_{i=1}^N$, raking estimates non-negative weights
$\{w_i\}$ such that the weighted feature totals
\[
\hat\alpha_j = \sum_i w_i f_j(x^{(i)})
\]
match the target aggregates $\alpha_j$, where $f_j$ are
typically indicator functions.
The algorithm iteratively processes each constraint,
rescaling the weights multiplicatively by a factor
$\alpha_j / \hat\alpha_j$ for individuals matching the
constraint, and cycles through all constraints until
convergence.

Raking has two structural limitations that motivate the
present work.

\paragraph{Degeneracy of the synthesised population.}
After convergence, the weight vector $\{w_i\}$ is
highly concentrated: a small fraction of individuals
accounts for almost all of the probability mass.
This is quantified by the \emph{effective sample size}
\begin{equation}
  \Neff = \left(\sum_{i=1}^N w_i^2\right)^{-1} \ll N.
\label{eq:neff}
\end{equation}
In our Syn-ISTAT experiments (\cref{sec:exp_diversity}),
raking achieves $\Neff \approx 0.012\,N$ with a Gini
coefficient of $0.951$ on the weight distribution---more
concentrated than any empirical income distribution.
When $N_{\text{synth}}$ individuals are drawn from the
reweighted sample for downstream simulation, most
individuals are repeated clones of a few archetypes,
reducing demographic diversity by a factor of $87$
relative to GibbsPCDSolver.
This matters for agent-based urban models where
behavioural heterogeneity is a first-order modelling
concern.

Raking does not generate new individuals but only
reweights a fixed sample, so diversity is entirely
limited by the support of the original data.

\paragraph{Degradation at large $K$.}
Because raking operates on a \emph{fixed} sample, it
cannot generate attribute combinations absent from the
original $N$ individuals.
As $K$ increases, the sample becomes increasingly sparse
relative to the full combinatorial space, limiting the
ability of the reweighting procedure to match overlapping
constraints.

Empirically, \citet{pachet2026} observe that the accuracy
of raking degrades significantly in high-dimensional
settings: for $K{=}40$ with ternary constraints, mean
relative error reaches approximately $21\%$, compared to
less than $5\%$ for MaxEnt.
This behaviour can be interpreted as the result of
sequential weight updates under strongly overlapping
constraints, which become increasingly difficult to
satisfy simultaneously.

GibbsPCDSolver avoids this limitation by sampling
directly from the model distribution, rather than
reweighting a fixed sample. As a result, it can in
principle assign positive probability to all
configurations in $\cX$, including combinations
absent from the original data.

\subsection{Deep Generative Models}
\label{sec:bg_deep}

A growing body of work applies deep generative models to
population synthesis.
Variational autoencoders \citep{borysov2019} learn a
low-dimensional latent representation of individual
profiles and generate new individuals by decoding
latent samples.
More recent approaches use normalising flows and
diffusion models to capture complex multimodal
distributions over categorical attributes
\citep{kim2023}.

All these methods share a common requirement: they are
trained on a \emph{microdata} sample of individual-level
records.
In many national statistical contexts only full or partial
aggregate contingency tables are published, and individual
microdata are restricted or unavailable.
The MaxEnt framework, by contrast, is designed precisely
for this aggregate-only regime: it requires no individual
records and is fully determined by the published marginals.
Extending MaxEnt to large $K$ via PCD, rather than
adopting deep generative models, is therefore the
principled choice for our application context.
The integration of deep generative models with
aggregate-constrained synthesis is an open research
direction that we leave to future work.

\subsection{Contrastive Divergence and Persistent CD}
\label{sec:bg_pcd}
Contrastive Divergence (CD) \citep{hinton2002} is a
family of algorithms for estimating the gradient of the
log-likelihood of an energy-based model without explicitly
computing the partition function. For models of the form
$p_\lambda(x) \propto \exp(\lambda \cdot F(x))$, the
gradient of the expected log-likelihood takes the form
\[
\nabla_\lambda\, \mathbb{E}_{\text{data}}[\log p_\lambda(x)]
= \mathbb{E}_{\text{data}}[F] - \mathbb{E}_{p_\lambda}[F],
\]
where the model expectation is generally intractable.
Standard CD-$k$ initialises a Gibbs chain from a data
point, runs $k$ steps, and uses the resulting sample to
approximate the model expectation. CD-$k$ introduces a
bias that decreases with $k$ but does not vanish in
general; in practice, even $k{=}1$ often works well for
restricted Boltzmann machines \citep{hinton2002}.

In our setting, however, no individual-level data are
available, and the "data expectation" is replaced by
target aggregate constraints. This requires adapting CD
to operate directly on the model distribution rather than
on observed samples.

Persistent Contrastive Divergence (PCD) \citep{tieleman2008}
eliminates the reinitialisation from data by maintaining
a persistent Markov chain across parameter updates.
The chain is \emph{not} reset at each gradient step;
instead, it continues from its previous state, thereby
tracking the evolving model distribution as $\blam$
changes.

In contrast to CD-$k$, this greatly reduces the bias of
the gradient estimate, since the samples are no longer
drawn from short-run Markov chains initialised at data
points. Under appropriate mixing conditions, the chain
provides samples that more closely approximate
expectations under $p_{\blam}$.

PCD is closely related to Stochastic Maximum Likelihood
\citep{younes1999} and can be analysed within the
framework of stochastic approximation. Convergence to a
local optimum can be established under standard
conditions on step sizes and mixing assumptions
\citep{robbins1951}.

In the context of the MaxEnt population synthesis problem,
PCD approximates the intractable expectation
\[
\sum_{\bx \in \cX} p_{\blam}(\bx) f_j(\bx)
\]
with an empirical average over a persistent Monte Carlo
sample of $N$ synthetic individuals.
This sample is updated across iterations and serves as a
proxy for draws from $p_{\blam}$.

The population is advanced by Gibbs (Glauber) dynamics
\citep{glauber1963,geman1984}, exploiting the factored
conditional structure of \cref{eq:gibbs_cond}, which
allows efficient local updates of individual attributes.

The partition function $Z(\blam)$ is never computed.

\section{Method: GibbsPCDSolver}
\label{sec:method}

\subsection{Setup and Notation}
\label{sec:notation}

Let $\cX = D_1 \times \cdots \times D_K$ be a finite product
space of $K$ categorical attributes, where attribute $k$ takes
values in a domain of size $d_k$.
A \emph{constraint} $j \in \{1,\ldots,m\}$ is a pattern indicator
\begin{equation}
  f_j(\bx) = \mathbf{1}\!\left[\bx_{S_j} = \mathbf{v}_j\right],
  \qquad S_j \subseteq \{1,\ldots,K\},\;
  \mathbf{v}_j \in \prod_{k \in S_j} D_k,
\end{equation}
with empirical target $\alpha_j  \in [0,1]$.
The \emph{maximum entropy} (MaxEnt) distribution over $\cX$
subject to the constraints
$\mathbb{E}_{p}[f_j] = \alpha_j$ for all $j$ is the exponential
family
\begin{equation}
  p_{\blam}(\bx)
    = \frac{1}{Z(\blam)}
      \exp\!\left(\sum_{j=1}^m \lambda_j f_j(\bx)\right),
  \qquad
  Z(\blam) = \sum_{\bx \in \cX} \exp\!\left(\sum_j \lambda_j f_j(\bx)\right).
\label{eq:maxent}
\end{equation}
The Lagrange multipliers $\blam \in \mathbb{R}^m$ are the
natural parameters of the model; they are the solution of the dual optimisation problem,
which is strictly convex under standard conditions.
the \emph{dual optimisation problem}
\begin{equation}
  \min_{\blam \in \mathbb{R}^m}
  \Phi(\blam) \coloneqq \log Z(\blam) - \blam^\top \balpha,
\label{eq:dual}
\end{equation}
whose gradient is $\nabla\Phi(\blam) = \bhatalpha(\blam) - \balpha$,
where
$\hat\alpha_j(\blam) = \mathbb{E}_{p_{\blam}}[f_j]
= \sum_{\bx \in \cX} p_{\blam}(\bx) f_j(\bx)$
are the model expectations.
The dual objective $\Phi$ is convex and, under standard
conditions, admits a unique minimiser $\blam^*$ satisfying
the moment-matching conditions
$\bhatalpha(\blam^*) = \balpha$.
In our setting, the targets $\balpha$ are specified
directly from aggregate data, and no individual-level
samples from $p_{\blam}$ are available.

\paragraph{The computational bottleneck.}
Computing $\nabla\Phi$ exactly requires a sum over all $|\cX|$
tuples.  For $K{=}15$ attributes with domains of size $2$--$5$,
$|\cX| \approx 1.7 \times 10^8$, which is at the limit of exact
enumeration.  For $K \ge 20$ with realistic domain sizes,

$|\cX| \gtrsim 10^{13}$ and exact computation becomes computationally infeasible in practice.

\chg{\citet{pachet2026} acknowledge this in their Section~5.1 and
suggest the use of sampling methods such as Gibbs or
Metropolis--Hastings to approximate expectations;
the present paper develops this direction.}{\citet{pachet2026} address this in their Section~5.1 by using
a Metropolis-type MCMC scheme to approximate expectations
in their implementation;
the present paper develops the complementary Gibbs route,
which exploits closed-form conditionals of the exponential
family to avoid rejection steps and proposal tuning,
and combines it with persistent chain maintenance~(PCD)
and an Adam optimiser robust to stochastic gradients.}

\paragraph{Forward map.}
The relationship between the natural parameters, the
unnormalised log-energies, and the model expectations can be
written as
\begin{equation}
  \blam
  \;\xrightarrow{\;\bF\blam\;}
  \mathbf{u}
  \;\xrightarrow{\;\exp\;}
  \mathbf{w}
  \;\xrightarrow{\;\mathrm{normalize}\;}
  p_{\blam}
  \;\xrightarrow{\;\mathbb{E}[\bF]\;}
  \bhatalpha(\blam),
\label{eq:adjoint}
\end{equation}
where $\bF \in \{0,1\}^{|\cX| \times m}$ is the indicator
matrix with $F_{\bx,j} = f_j(\bx)$,
and $\mathbf{u}, \mathbf{w} \in \mathbb{R}^{|\cX|}$ are
vectors indexed by tuples $\bx \in \cX$, with
$u_{\bx} = \sum_j \lambda_j f_j(\bx)$ the log-unnormalised
energy of $\bx$.
GibbsPCDSolver bypasses the final three steps entirely:
rather than materialising $\mathbf{w}$ or computing
$Z(\blam) = \sum_{\bx} w_{\bx}$, it draws samples from
$p_{\blam}$ directly via persistent Gibbs dynamics and
estimates $\bhatalpha$ from the resulting pool.

\subsection{Algorithm}
\label{sec:algorithm}

\paragraph{Persistent Contrastive Divergence.}
PCD \citep{tieleman2008} maintains a \emph{persistent pool} $\mathbf{P} \in \cX^N$
of $N$ synthetic individuals.  This yields a Monte Carlo estimate of the model
expectations $\mathbb{E}_{p_{\blam}}[f_j]$.

At each outer iteration, the pool is advanced
by $s$ sweeps of Gibbs (Glauber) dynamics, and the resulting
empirical frequencies
$\hat\alpha_j = N^{-1}\sum_{i=1}^N f_j(\mathbf{P}[i])$
are used as a stochastic approximation of $\nabla\Phi$.
The partition function $Z(\blam)$ is \emph{never} computed.

The key ingredient is the conditional distribution of attribute
$k$ given all others, which can be computed exactly from the
exponential family.

\begin{proposition}[Gibbs conditionals]
\label{prop:gibbs}
Under $p_{\blam}$, the conditional distribution of $A_k$
given the remaining attributes $\bx_{-k}$ is
\begin{equation}
  p_{\blam}(A_k = v \mid \bx_{-k})
  \;\propto\;
  \exp\!\bigl(
    \sum_{j \in \mathcal{J}(k, v, \bx_{-k})}
    \lambda_j
  \bigr),
\label{eq:gibbs_cond}
\end{equation}
where $\mathcal{J}(k, v, \bx_{-k})$ denotes the set of
constraints $j$ such that $k \in S_j$, $v^{(k)}_j = v$,
and $\bx_{S_j \setminus \{k\}} = \mathbf{v}^{(-k)}_j$,
i.e.\ all remaining attributes in the constraint are
consistent with $\bx_{-k}$.
\end{proposition}

\begin{proof}
From \cref{eq:maxent}, we have
$p_{\blam}(\bx) \propto \exp\!\bigl(\sum_j \lambda_j f_j(\bx)\bigr)$.
Fixing $\bx_{-k}$, only the terms with $k \in S_j$ depend on
$x_k$. Therefore,
\[
p_{\blam}(A_k = v \mid \bx_{-k})
\;\propto\;
\exp\!\Bigl(
\sum_{j:\, k \in S_j} \lambda_j f_j(\bx_{-k}, v)
\Bigr),
\]
and the result follows by observing that $f_j(\bx_{-k}, v)=1$
if and only if $j \in \mathcal{J}(k, v, \bx_{-k})$.
Normalisation over $v \in D_k$ yields the categorical
distribution in \cref{eq:gibbs_cond}.
\end{proof}

The update rule of \cref{prop:gibbs} leaves $p_{\blam}$ invariant
by detailed balance: for any two configurations differing only in
attribute $k$, the acceptance is automatic (no rejection step).
Ergodicity follows from the fact that $p_{\blam}(\bx) > 0$ for all
$\bx \in \cX$---a consequence of the exponential family having full
support---so the chain is irreducible and aperiodic.
Crucially, the normalising constant of the conditional
\begin{equation*}
  Z_k(v, \bx_{-k})
  = \sum_{v' \in D_k}
    \exp\!\Bigl(\sum_{j \in \mathcal{J}(k,v',\bx_{-k})} \lambda_j\Bigr)
\end{equation*}
requires a sum over $d_k$ terms only, replacing the intractable
sum over $|\cX|$ tuples in the global partition function.
Each attribute is updated in a randomly permuted order
within each sweep, breaking cyclic correlations between
attributes.

To enable vectorised accumulation of log-energies over the
$N$-element pool,
we precompute for each attribute $k$ the list of all constraints
that involve $k$, together with their context requirements:
\begin{equation}
  \mathrm{lookup}[k]
  = \bigl\{\,(j,\; v^{(k)}_j,\; S_j \setminus \{k\},\;
             \mathbf{v}^{(-k)}_j) \;:\; k \in S_j\,\bigr\}.
\end{equation}
Each entry records the constraint index $j$, the value
$v^{(k)}_j$ that attribute $k$ must take, and the values
$\mathbf{v}^{(-k)}_j$ that the remaining attributes
$S_j \setminus \{k\}$ must match for the constraint to be active.
A single Gibbs sweep over attribute $k$ then computes log-energies
of shape $(N, d_k)$ in $O(N \cdot |\mathrm{lookup}[k]|)$ time,
applies a numerically stable softmax, and samples the new value
for each of the $N$ individuals simultaneously.

\Cref{alg:gibbspcd} presents the complete procedure.
The outer loop (lines \ref{alg:outer_start}--\ref{alg:outer_end})
updates $\blam$ via Adam; the inner loop
(lines \ref{alg:inner_start}--\ref{alg:inner_end}) advances the
persistent pool.
Line~\ref{alg:expectation} is the only place where the
approximation is introduced: $\bhatalpha$ is estimated from the
pool rather than computed exactly.

\begin{algorithm}[t]
\caption{GibbsPCDSolver}
\label{alg:gibbspcd}
\begin{algorithmic}[1]
\Require constraint set $\{(S_j, \mathbf{v}_j, \alpha_j)\}_{j=1}^m$,
         domain sizes $\{d_k\}$,
         pool size $N$, sweeps per step $s$,
         learning rate $\eta$, outer iterations $T$
\State $\blam \gets \mathbf{0}$;\;
       initialise pool $\mathbf{P}[i,k] \sim \mathrm{Uniform}(D_k)$
\State precompute $\mathrm{lookup}[k]$ for all $k$
\For{$t = 1, \ldots, T$} \label{alg:outer_start}
  \For{sweep $= 1, \ldots, s$} \label{alg:inner_start}
    \For{$k \in \mathrm{shuffle}(\{1,\ldots,K\})$}
      \State accumulate log-energies $\mathbf{E} \in \mathbb{R}^{N \times d_k}$
             via $\mathrm{lookup}[k]$ and $\blam$
      \State $\mathbf{P}[\,\cdot\,, k]
             \sim \mathrm{Categorical}(\mathrm{softmax}(\mathbf{E}))$
    \EndFor
  \EndFor \label{alg:inner_end}
  \State $\hat\alpha_j \gets N^{-1}\sum_i f_j(\mathbf{P}[i])$
         for all $j$ \label{alg:expectation}
  \State $\mathbf{g} \gets \bhatalpha - \balpha$ \Comment{stochastic gradient of $\Phi$}
  \State $\blam \gets \blam - \eta_t\,\mathrm{Adam}(\mathbf{g}, t)$
         \label{alg:adam}
  \If{stopping rule satisfied} \textbf{break} \EndIf
\EndFor \label{alg:outer_end}
\State \Return $\blam$, final pool $\mathbf{P}$
\end{algorithmic}
\end{algorithm}

Let $\bar{J}$ be the mean number of constraints per attribute and
$\bar\ell$ the mean pattern length.
The cost of one full Gibbs sweep can be bound as
$O(N \cdot K \cdot d_{\max} \cdot \bar{J} \cdot \bar\ell)$,
growing linearly in $N$ and $K$.
Crucially, neither $|\cX|$ nor $Z(\blam)$ appears in this
expression.

The only modification with respect to Algorithm~1 of
\citet{pachet2026} is the replacement of exact expectation
computation (their line~4) with the PCD estimate
at line~\ref{alg:expectation} of \cref{alg:gibbspcd}. This substitution enables scalability to settings where
exact enumeration of $\cX$ is infeasible.
All other components---the exponential family, the dual
objective, the constraint structure---are unchanged.

\subsection{Optimiser, Stopping Rule, and Practical Guidelines}
\label{sec:practical}

\paragraph{Adam optimiser.}
We use Adam \citep{kingma2015} with default parameters
$\beta_1{=}0.9$, $\beta_2{=}0.999$, $\varepsilon{=}10^{-8}$.
Adam stabilises optimisation under noisy, non-stationary
stochastic gradients (whose variance is $O(1/N)$)
and consistently outperforms plain SGD in our experiments (\cref{sec:exp_sensitivity}).

\paragraph{Adaptive stopping rule.}
Let $\mathrm{MRE}_t = m^{-1}\sum_j |\hat\alpha_j^t - \alpha_j|/\alpha_j$
be the mean relative error at iteration $t$.
We stop when the relative improvement of the MRE minimum over
two consecutive windows of length $W$ falls below a threshold
$\tau$:
\begin{equation}
  \frac{\min_{t \in [T-2W,\, T-W]} \MRE_t
        - \min_{t \in [T-W,\, T]} \MRE_t}
       {\min_{t \in [T-2W,\, T-W]} \MRE_t}
  < \tau.
\label{eq:stopping}
\end{equation}
Using the minimum (rather than the mean) over each window makes
the rule robust to the stochastic fluctuations of the gradient
estimate.

\paragraph{Practical guidelines.}
\Cref{tab:guidelines} summarises the hyperparameter
recommendations derived from the experiments in
\cref{sec:experiments}.
The key finding is that the optimal number of Gibbs sweeps $s$
scales with the complexity of the constraint graph, not with $K$
alone: $s{=}1$ suffices for sparse pairwise constraints at
$K{=}8$, but $s{=}5$ is necessary for mixed-arity (1+2+3-way)
constraints at $K{=}15$.

\begin{table}[t]
\centering
\caption{Practical hyperparameter guidelines for GibbsPCDSolver.}
\label{tab:guidelines}
\begin{tabular}{lll}
\toprule
Parameter & Recommendation & Rationale \\
\midrule
Pool size $N$
  & $25\,000$--$100\,000$
  & MRE floor $\approx 1/(2\sqrt{N})$; sweet spot at $N{=}25$K \\
Sweeps $s$
  & $1$ (sparse, arity $\le 2$);
    $5$ (mixed arity $\le 3$)
  & Extra sweeps reduce mixing bias (not gradient variance) \\
Learning rate $\eta$
  & $0.01$ for $K \ge 15$
  & $\eta{=}0.05$ oscillates at $K{=}15$ with arity-3 constraints \\
Window $W$
  & $50$
  & Longer window avoids premature stopping during slow descent \\
Initialisation
  & $\blam = \mathbf{0}$, pool uniform
  & Max-entropy prior; data-based initialisation may introduce CD bias \\
\bottomrule
\end{tabular}
\end{table}

These hyperparameters control a trade-off between variance
($N$), mixing bias ($s$), and optimisation stability ($\eta$).

\subsection{Algorithmic Validation}
\label{sec:validation}

We first validate that GibbsPCDSolver converges to the exact MaxEnt
solution when the latter can be computed exactly (e.g., by
enumeration for small $K$).
The three experiments below use synthetic data generated by the
\textbf{WuGenerator}, a controlled benchmark generator we
implemented following the experimental protocol of \citet{wu2016},
which constructs a ground truth as follows:
(i) each attribute $k$ is assigned a random marginal
distribution over $D_k$;
(ii) a set of $P$ binary or ternary \emph{planted patterns}
$(S_p, \mathbf{v}_p, \nu_p)$ is drawn, each with a target
frequency $\nu_p \in (0.05, 0.35)$;
(iii) each synthetic individual independently activates each
pattern with probability $\nu_p$ (Bernoulli draw), with
overlaps resolved by a fixed rule, receiving the prescribed
values for $\bx_{S_p}$, while attributes not covered by any
activated pattern are sampled from their marginals.
The constraint set is then defined as the empirical
frequencies computed from $N_{\text{data}}$ such individuals.
This procedure yields, by construction, a consistent set of
constraints with known planted structure, making it
well suited for algorithmic validation.

\paragraph{Experiment A0 — Gibbs conditionals ($K{=}6$).}

\textit{Setup.}
A WuGenerator instance with $K{=}6$,
domain sizes $[3,3,3,2,2,2]$, $|\cX|{=}216$, and $P{=}3$
planted binary patterns (frequencies $0.251$, $0.235$, $0.349$)
produces a constraint set with $m{=}33$ constraints
(15 unary, 18 binary).
Targets are extracted from $N_{\text{data}}{=}100{,}000$
samples.
The exact L-BFGS solver reaches $\MRE{=}3{\times}10^{-5}$
in $27$ iterations ($0.06$~s).

\textit{Results.}
\Cref{tab:a0} reports GibbsPCDSolver at four pool sizes
with $s{=}3$ sweeps.

\begin{table}[h]
\centering
\caption{%
  Experiment A0 ($K{=}6$, $|\cX|{=}216$, $m{=}33$,
  $s{=}3$, $\eta{=}0.05$).
  Exact L-BFGS: $\MRE{=}3{\times}10^{-5}$, $t{=}0.06$~s.
  $\|\Delta\blam\|$: relative parameter distance
  $\|\blam_{\mathrm{MCMC}}{-}\blam^*\|/\|\blam^*\|$.}
\label{tab:a0}
\begin{tabular}{rrrrr}
\toprule
$N$ & MRE & $\|\Delta\blam\|/\|\blam^*\|$ & KL & Time (s) \\
\midrule
$500$      & 0.101 & 0.344 & $5.9{\times}10^{-3}$ & 0.6 \\
$2{,}000$  & 0.057 & 0.341 & $1.9{\times}10^{-3}$ & 1.2 \\
$10{,}000$ & 0.021 & 0.349 & $4.8{\times}10^{-4}$ & 6.1 \\
$100{,}000$& 0.010 & 0.347 & $5.6{\times}10^{-5}$ & 86.7 \\
\bottomrule
\end{tabular}
\end{table}

MRE and KL decrease monotonically with $N$, consistent with
the $O(1/\sqrt{N})$ and $O(1/N)$ theoretical rates.
The relative parameter distance $\|\Delta\blam\|/\|\blam^*\|$
remains approximately constant at $0.34$ across all pool
sizes; consistent with the expected behaviour.
Unary constraints leave a gauge freedom in $\blam$:%
\footnote{%
  Formally, for each attribute $k$ define the shift
  $\blam \mapsto \blam + c_k \mathbf{e}_k$, where
  $\mathbf{e}_k$ adds the constant $c_k$ to all
  $\lambda_j$ with $k \in S_j$ and $|S_j|=1$.
  Since $\sum_{v \in D_k} p_{\blam}(A_k{=}v \mid \bx_{-k})
  = 1$ for every $\bx_{-k}$, this shift leaves all
  conditional distributions --- and hence $p_{\blam}$
  itself --- unchanged.
  The L-BFGS solver returns the unique $\blam^*$ with
  minimum $\ell_2$ norm; GibbsPCDSolver converges to
  a different element of the same equivalence class.
  The distance $\|\hat{\blam} - \blam^*\|/\|\blam^*\|$
  measures the gap between two representatives, not
  solver error.}
shifting all $\lambda_j$ involving attribute $k$ by a
constant does not change $p_{\blam}$, so the distance
$\|\hat{\blam} - \blam^*\|/\|\blam^*\|$ is not a
meaningful convergence metric; MRE and KL are the
correct diagnostics.
(shifting all $\lambda_j$ involving attribute $k$ by a
constant does not change $p_{\blam}$), so the distance to the
specific $\blam^*$ returned by L-BFGS is not a meaningful
convergence metric; MRE and KL are the correct diagnostics.
\Cref{fig:a0} shows convergence curves, a $\lambda$ scatter
plot, and the constraint frequency scatter $\hat\alpha_j$
vs.\ $\alpha_j$, all confirming correct behaviour.

\begin{figure}[t]
  \centering
  \includegraphics[width=\textwidth]{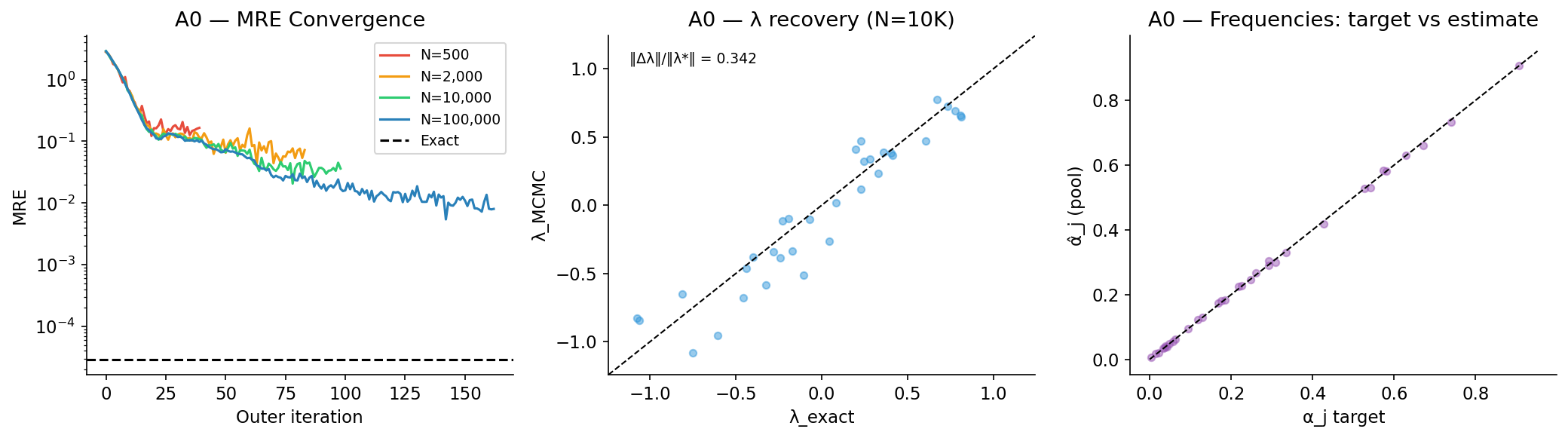}
  \caption{%
    Experiment A0 ($K{=}6$, $|\cX|{=}216$).
    \emph{Left}: MRE convergence curves for four pool sizes
    (log scale); dashed line is exact L-BFGS reference.
   \emph{Centre}: scatter of $\hat{\blam}$ vs.\ $\blam^*$
at $N{=}10{,}000$; the systematic offset from the diagonal
reflects a gauge degree of freedom in $\blam$ (see footnote),
not solver error.
    \emph{Right}: estimated frequencies $\hat\alpha_j$
    vs.\ targets $\alpha_j$ at $N{=}10{,}000$;
    points lie close to the diagonal across the full
    range $[0.004,\, 0.906]$.}
  \label{fig:a0}
\end{figure}

\paragraph{Experiment A1a --- Wu benchmark ($K{=}8$).}

\textit{Setup.}
A WuGenerator instance with $K{=}8$,
domain sizes $[4,4,2,2,4,4,3,2]$, $|\cX|{=}6{,}144$,
$P{=}4$ planted binary patterns (frequencies
$0.095$--$0.283$), and $N_{\text{data}}{=}200{,}000$
samples produces a constraint set with $m{=}61$
constraints (25 unary, 36 binary).
The exact L-BFGS solver reaches $\MRE{=}4{\times}10^{-5}$
in $0.03$~s.

\textit{Results.}
\Cref{tab:validation} (A1a rows) and \cref{fig:a1a}
report GibbsPCDSolver across five configurations.
MRE decreases from $0.099$ at $N{=}1{,}000$ to
$0.011$ at $N{=}50{,}000$, tracking the $O(1/\sqrt{N})$
floor closely (panel~c).
KL divergence scales as $O(1/N)$, from
$9.0{\times}10^{-3}$ to $1.8{\times}10^{-4}$.
As in A0, the relative parameter distance
$\|\Delta\blam\|/\|\blam^*\| \approx 0.50$ is
insensitive to $N$, confirming that it reflects
gauge non-identifiability rather than solver error.
Doubling sweeps from $s{=}5$ to $s{=}10$ at
$N{=}20{,}000$ yields no MRE improvement
($0.022$ vs $0.023$) at $2.3{\times}$ the runtime,
confirming that variance dominates over mixing bias
at this pool size.

\begin{figure}[t]
  \centering
  \includegraphics[width=\textwidth]{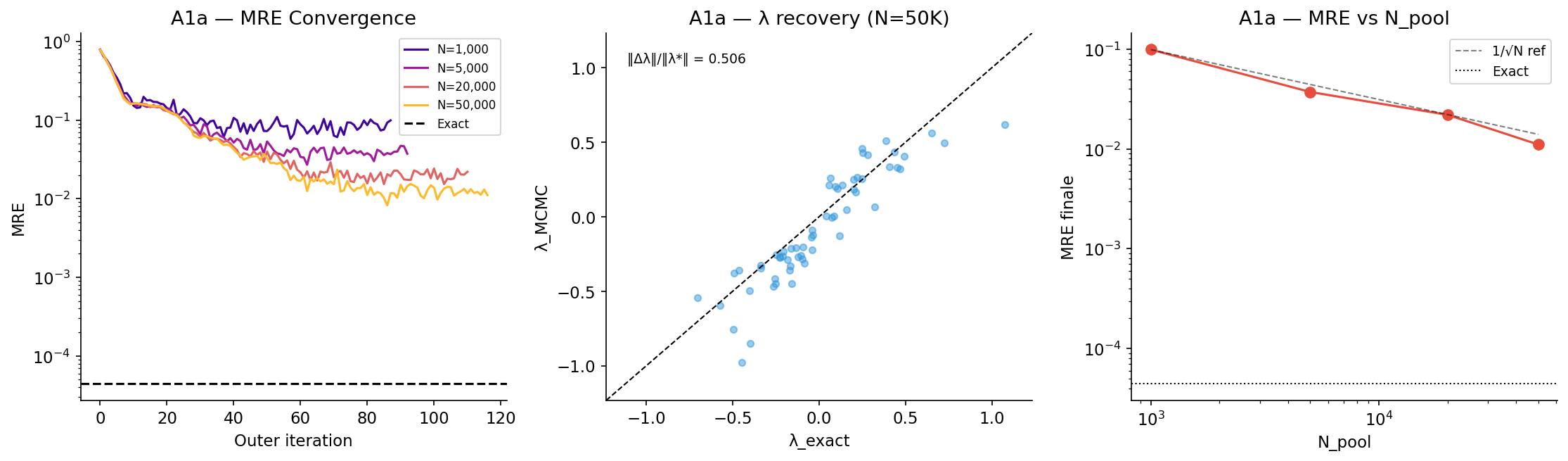}
  \caption{%
    Experiment A1a ($K{=}8$, $|\cX|{=}6{,}144$,
    $m{=}61$, $s{=}5$).
    \emph{Left}: MRE convergence curves for four pool
    sizes (log scale); dashed line is the exact
    L-BFGS reference ($\MRE{=}4{\times}10^{-5}$).
    \emph{Centre}: scatter of $\blam_{\mathrm{MCMC}}$
    vs.\ $\blam^*$ at $N{=}50{,}000$; the constant
    offset from the diagonal is the gauge degree of
    freedom induced by unary constraints and does not
    affect the distribution $p_{\blam}$.
    \emph{Right}: final MRE vs.\ $N$ (log--log);
    the dashed reference line follows $1/\sqrt{N}$,
    confirming the theoretical variance floor.}
  \label{fig:a1a}
\end{figure}

\paragraph{Experiment A1b --- Planted exponential family ($K{=}10$).}

\textit{Setup.}
The PlantedExpFamilyGenerator draws $\blam^*$ explicitly
from a uniform distribution over $[-1.07, 1.01]$ and
constructs $p_{\blam^*}$ by exact enumeration over
$|\cX|{=}59{,}049$ tuples ($K{=}10$, all domains of
size~3).
The constraint set consists of $m{=}30$ binary constraints
(no unary), with targets $\alpha^*_j = \mathbb{E}_{p_{\blam^*}}[f_j]$
computed exactly.
This design eliminates gauge non-identifiability
(no unary constraints) and allows exact computation of
$\KL(p_{\blam^*} \| p_{\blam_{\mathrm{MCMC}}})$ and
$\|\blam_{\mathrm{MCMC}} - \blam^*\| / \|\blam^*\|$.
The model entropy is $H(p_{\blam^*}){=}10.57$~nats
(vs.\ $H_{\max}{=}\log 59{,}049 \approx 11.0$~nats),
indicating a moderately constrained distribution.
The exact L-BFGS solver recovers $\blam^*$ to
$\|\Delta\blam\|/\|\blam^*\|{=}8{\times}10^{-5}$
with $\KL{<}10^{-6}$.

\textit{Results.}
\Cref{tab:validation} (A1b rows) and \cref{fig:a1b}
report GibbsPCDSolver across four pool sizes.
Because unary constraints are absent, the parameter
distance is now a meaningful convergence metric:
$\|\Delta\blam\|/\|\blam^*\|$ decreases from $0.072$
at $N{=}1{,}000$ to $0.016$ at $N{=}50{,}000$,
tracking the same $O(1/\sqrt{N})$ rate as MRE.
The KL divergence scales as $O(1/N)$
(from $2.5{\times}10^{-3}$ to $1.2{\times}10^{-4}$),
and the $\lambda$ scatter at $N{=}50{,}000$
(panel~b of \cref{fig:a1b}) lies tightly around
the diagonal with no systematic offset,
confirming that GibbsPCDSolver recovers the true
parameters when the problem is identifiable.

\begin{figure}[t]
  \centering
  \includegraphics[width=\textwidth]{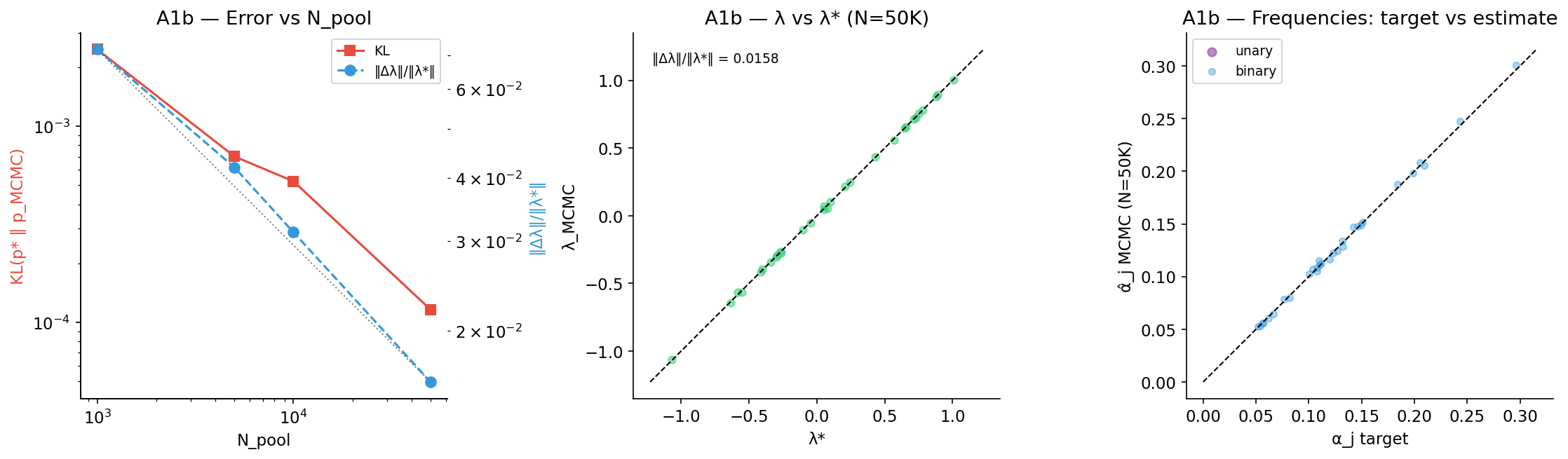}
  \caption{%
    Experiment A1b ($K{=}10$, $|\cX|{=}59{,}049$,
    $m{=}30$ binary constraints, $s{=}5$).
    \emph{Left}: KL divergence $\KL(p_{\blam^*}\|p_{\blam_{\mathrm{MCMC}}})$
    (red, left axis) and relative parameter distance
    $\|\Delta\blam\|/\|\blam^*\|$ (blue, right axis)
    vs.\ pool size $N$ (log--log); the dotted line
    shows the $O(1/N)$ reference.
    \emph{Centre}: scatter of $\blam_{\mathrm{MCMC}}$
    vs.\ $\blam^*$ at $N{=}50{,}000$; unlike A0--A1a,
    no gauge offset is present and points cluster
    tightly around the diagonal
    ($\|\Delta\blam\|/\|\blam^*\|{=}0.016$).
    \emph{Right}: estimated constraint frequencies
    $\hat\alpha_j$ vs.\ exact targets $\alpha^*_j$
    at $N{=}50{,}000$, for binary constraints (blue).}
  \label{fig:a1b}
\end{figure}
\begin{table}[t]
\centering
\caption{%
  Algorithmic validation: GibbsPCDSolver vs.\ exact L-BFGS.
  MRE: mean relative error on training constraints.
  $\|\Delta\blam\|$: relative parameter distance
  (computed only for A1b where $\blam^*$ is known exactly).
  KL: $\KL(p_{\blam^*} \| p_{\blam_{\mathrm{MCMC}}})$
  (computed only when $|\cX|$ is enumerable).}
\label{tab:validation}
\begin{tabular}{llrrrrr}
\toprule
Exp. & Method & $N$ & $s$ & MRE & $\|\Delta\blam\|/\|\blam^*\|$ & KL \\
\midrule
A1a   & Exact L-BFGS      &  ---    & --- & $4\times10^{-5}$ & ---    & ---      \\
(K=8, & Gibbs             &  1,000  &   5 & 0.0991           & 0.498  & 0.00898  \\
$|\cX|$& Gibbs            &  5,000  &   5 & 0.0373           & 0.505  & 0.00098  \\
$=$6144)& Gibbs           & 20,000  &   5 & 0.0220           & 0.503  & 0.00029  \\
       & Gibbs            & 50,000  &   5 & 0.0111           & 0.506  & 0.00018  \\
       & Gibbs            & 20,000  &  10 & 0.0231           & 0.505  & 0.00049  \\
\midrule
A1b    & Exact L-BFGS      &  ---    & --- & $2\times10^{-5}$ & $8\times10^{-5}$ & $<10^{-6}$ \\
(K=10, & Gibbs             &  1,000  &   5 & 0.0906           & 0.0719  & $2.5\times10^{-3}$ \\
$|\cX|$& Gibbs             &  5,000  &   5 & 0.0401           & 0.0420  & $7.0\times10^{-4}$ \\
$=$59049)& Gibbs           & 10,000  &   5 & 0.0298           & 0.0313  & $5.2\times10^{-4}$ \\
       & Gibbs             & 50,000  &   5 & 0.0180           & 0.0158  & $1.2\times10^{-4}$ \\
\bottomrule
\end{tabular}
\end{table}

\section{The Syn-ISTAT Benchmark}
\label{sec:synistat}

We introduce \emph{Syn-ISTAT}, a synthetic demographic benchmark
designed to evaluate MaxEnt solvers in the regime where exact
enumeration is infeasible ($K{=}15$, $|\cX| \approx 1.7 \times 10^8$)
and individual microdata are unavailable.

Syn-ISTAT is an \emph{ISTAT-inspired} benchmark, {\it not} a dataset
derived from official Italian microdata.
Its conditional probability tables (CPTs) reflect known
structural patterns of the Italian demographic system---the
gender gap in labour market participation, age-dependent
income profiles, the urban/rural transport split---but are
constructed synthetically to ensure full ground-truth control.
This design choice is deliberate: it allows exact analytical
computation of all marginal targets and eliminates any
confound between solver error and data inconsistency.
The application of GibbsPCDSolver to real ISTAT municipal
data and the associated
challenges of heterogeneous constraint universes and
inconsistent reference populations are left to a companion
paper currently in preparation.

Syn-ISTAT is a Bayesian network over $K{=}15$ categorical
attributes representing the demographic and socioeconomic
profile of a synthetic Italian individual:
sex, age, marital status, education, employment, income,
household size, presence of children, residential area,
car access, main transport mode, commute time, diet type,
alcohol use, and physical activity
(full attribute table in \cref{app:synistat}).
Six \emph{anchor} variables (sex, age, education,
household size, residential area, car access) are drawn
from fixed marginals; the remaining nine are sampled from
CPTs conditioned on one or two anchor variables, forming a
shallow Bayesian network with limited dependency depth.
Three ternary CPTs encode higher-order interactions that
are invisible to any pair of binary tables:
$P(\text{employment} \mid \text{sex}, \text{age})$,
$P(\text{income} \mid \text{education}, \text{employment})$, and
$P(\text{transport} \mid \text{area}, \text{car access})$.

All marginal targets $\alpha^*_j$ are computed
\emph{analytically} by factor marginalisation over the CPTs,
without Monte Carlo sampling.
This ensures that any solver error measured against $\alpha^*$
reflects the algorithm's approximation, not sampling noise.
For a binary constraint $(k_1, k_2)$:
\begin{equation}
  \alpha^*_{k_1 k_2}(v_1, v_2)
  = \sum_{\bx_{-\{k_1,k_2\}}}
    p_{\mathrm{BN}}(\bx),
\end{equation}
computed via \texttt{np.einsum} contractions following the
topological order of the network.
All 31 marginals normalise to $1$ within machine precision
($< 10^{-14}$).

The full constraint set consists of $m_1{=}15$ unary,
$m_2{=}13$ binary, and $m_3{=}3$ ternary constraints,
expanded to $m \approx 280$ atomic indicator features.
For the held-out experiment (\cref{sec:exp_heldout}),
we use the 28 unary and binary constraints as training
and the 3 ternary constraints as held-out evaluation:
the ternary targets encode interactions strictly beyond the
solver's training signal, providing a direct test of
generalisation to unseen higher-order interactions.


Full CPT specifications are reproduced in \cref{app:synistat}.

\section{Experiments}
\label{sec:experiments}

All experiments use the GibbsPCDSolver implementation
described in \cref{sec:method}, with Adam optimiser
($\beta_1{=}0.9$, $\beta_2{=}0.999$, $\varepsilon{=}10^{-8}$)
and adaptive stopping according to \cref{eq:stopping}
with $\tau{=}0.02$.
Unless otherwise noted, the persistent pool is initialised
uniformly and $\blam_0 = \mathbf{0}$.

As baseline, we use generalised raking (IPF) on an initial
synthetic sample of size $N$ drawn uniformly over the
attribute domains.
Its error is evaluated from the weighted frequency vector,
ensuring direct comparability with GibbsPCDSolver on the same
set of aggregate constraints.

\paragraph{Organisation of this section.}
The experiments are structured to address four complementary
questions in order of increasing specificity.
\Cref{sec:exp_scaling} tests the central scalability claim:
does GibbsPCDSolver maintain competitive accuracy as $|\cX|$
grows beyond the reach of exact enumeration?
\Cref{sec:exp_heldout} asks whether the learned model
generalises to higher-order interactions withheld from
training, probing the structural advantage of the exponential
family parametrisation over sequential IPF rescaling.
\Cref{sec:exp_diversity} addresses the question that is
invisible from MRE alone: does the synthesised population
provide the demographic variety required for agent-based
simulation?
\Cref{sec:exp_wu,sec:exp_sensitivity} close with
hyperparameter characterisation, moving from a controlled
setting with exact reference ($K{=}8$) to the realistic
Syn-ISTAT regime ($K{=}15$) where exact computation is
unavailable.

\paragraph{Note on training configurations.}
Two distinct training configurations appear in this section.
\Cref{sec:exp_scaling,sec:exp_diversity} train on all 31
Syn-ISTAT constraints (15 unary + 13 binary + 3 ternary),
reaching $\MRE{=}0.030$ at convergence (535 iterations,
$N{=}100{,}000$).
\Cref{sec:exp_heldout} withholds the 3 ternary constraints
and trains on the 28 unary and binary constraints only,
reaching $\MRE_{\text{train}}{=}0.018$ (600 iterations);
the held-out ternary constraints are evaluated exclusively
at test time.
Unless otherwise stated, MRE values refer to the
respective training set of each experiment.

\subsection{Scaling Beyond Exact Enumeration
            ($K \in \{12, 20, 30, 40, 50\}$)}
\label{sec:exp_scaling}

We begin by evaluating GibbsPCDSolver in the regime where
exact enumeration becomes infeasible.
This experiment directly tests the central claim of the paper:
that MaxEnt inference via PCD scales with $N$ and $K$,
rather than with $|\cX|$.

More precisely, we evaluate GibbsPCDSolver across five WuGenerator instances
with $K \in \{12, 20, 30, 40, 50\}$, exclusively ternary
planted patterns, and $|\cX|$ ranging from $83{,}000$
to $5.8{\times}10^{18}$.
For $K{=}12$ the exact L-BFGS solver is still available
and serves as reference; for $K \ge 20$, $|\cX|$ is not
enumerable and GibbsPCDSolver is the only MaxEnt-based
method that can operate.
All runs use $N{=}100{,}000$, $s{=}5$, $\eta{=}0.01$,
with Numba acceleration for $K \ge 20$.
We compare against generalised raking at the same $N$,
using MRE, $\Neff$, and Shannon entropy $H$ as metrics.
\Cref{tab:a2} reports the full results;
\cref{fig:a2_scaling} summarises the key trends.

\begin{table}[t]
\centering
\caption{%
  A2 scaling experiments: GibbsPCDSolver vs.\ raking
  across $K \in \{12, 20, 30, 40, 50\}$ with ternary
  constraints only, $N{=}100{,}000$, $s{=}5$.
  For $K{=}12$ the exact L-BFGS reference is available
  ($\MRE{=}1.3{\times}10^{-4}$, $t{=}2.4$~s);
  for $K \ge 20$, $|\cX|$ is not enumerable.
  Raking $\MRE{=}0$ for $K \le 30$ is exact on training
  constraints by algebraic construction and is not a
  meaningful accuracy metric; for $K \ge 40$ raking
  fails to satisfy the constraints ($\MRE > 0$).}
\label{tab:a2}
\begin{tabular}{llrrrrrr}
\toprule
$K$ & Method & $|\cX|$ & MRE & $t$ (s)
    & $\Neff$ & $\Neff/N$ & $H$ (nats) \\
\midrule
\multirow{2}{*}{12}
  & Gibbs PCD & $8.3{\times}10^4$  & 0.012 &  607 & 100,000 & 100\%  & 9.40 \\
  & Raking    & ---                & $0^\dagger$ &    7 &   7,138 &   7.1\% & 9.19 \\
\midrule
\multirow{2}{*}{20}
  & Gibbs PCD & $2.7{\times}10^7$  & 0.010 &  440 & 100,000 & 100\%  & 11.43 \\
  & Raking    & ---                & $0^\dagger$ &   16 &   3,102 &   3.1\% & 9.37 \\
\midrule
\multirow{2}{*}{30}
  & Gibbs PCD & $4.7{\times}10^{11}$ & 0.010 &  965 & 100,000 & 100\%  & 11.51 \\
  & Raking    & ---                  & $0^\dagger$ &  16 &     790 &   0.8\% & 8.36 \\
\midrule
\multirow{2}{*}{40}
  & Gibbs PCD & $2.8{\times}10^{16}$ & 0.018 & 1550 & 100,000 & 100\%  & 11.51 \\
  & Raking    & ---                  & 0.129$^\ddagger$ & 22 & 122 & 0.12\% & 5.09 \\
\midrule
\multirow{2}{*}{50}
  & Gibbs PCD & $5.8{\times}10^{18}$ & 0.013 & 2061 & 100,000 & 100\%  & 11.51 \\
  & Raking    & ---                  & 0.065$^\ddagger$ & 27 & 125 & 0.12\% & 5.14 \\
\bottomrule
\multicolumn{8}{l}{%
  $^\dagger$ Exact on training by construction (see text).\quad
  $^\ddagger$ Raking fails: training MRE $> 0$.}
\end{tabular}
\end{table}

\paragraph{MRE stability across five decades of $|\cX|$.}
The most striking result is that GibbsPCDSolver maintains
MRE in the range $[0.010, 0.018]$ across all five values
of $K$, while $|\cX|$ grows from $83{,}000$ to
$5.8{\times}10^{18}$---eighteen orders of magnitude.
This confirms that the computational cost of GibbsPCDSolver
scales with $N$ and $K$, not with $|\cX|$, and that
accuracy does not degrade as the problem grows
into the genuinely non-enumerable regime.
The runtime grows approximately as $O(K)$
(from $607$~s at $K{=}12$ to $2{,}061$~s at $K{=}50$),
consistent with the theoretical cost
$O(N \cdot K \cdot \bar{J} \cdot d_{\max})$ per sweep.

\paragraph{Raking breaks at $K{=}40$.}
For $K \le 30$, raking achieves $\MRE{=}0$ on training
constraints by algebraic construction.
At $K{=}40$, raking fails: $\MRE{=}0.129$, confirming
the competence-map threshold identified by
\citet{pachet2026} at $K \approx 28$--$40$.
At $K{=}50$, raking partially recovers to $\MRE{=}0.065$,
likely because the WuGenerator problem at $K{=}50$ has
a different constraint graph topology than the
NPORS-derived instances of \citet{pachet2026}.

\paragraph{Comparison with \citet{pachet2026}.}
\Cref{fig:a2_scaling} (left) overlays our GibbsPCDSolver
results with the exact MaxEnt results of \citet{pachet2026}
(their Table~2, arity$=3$, $N{\approx}100{,}000$).
Our MRE values at $K \in \{12, 20\}$ are numerically
close to theirs ($0.012$ vs.\ $0.017$ at $K{=}12$;
$0.010$ vs.\ $0.013$ at $K{=}20$).
\emph{This does not imply that GibbsPCDSolver outperforms
exact MaxEnt}: the two sets of experiments use different
problem instances (WuGenerator ternary patterns
vs.\ NPORS-derived constraints) with different
constraint densities and planted structure.
The comparison is qualitative and its purpose is to
show that GibbsPCDSolver achieves competitive accuracy
in the regime where exact enumeration applies, while
extending to $K{=}50$ where $|\cX| \approx 6{\times}10^{18}$
makes exact computation impossible by any measure.

\paragraph{Raking degeneracy grows monotonically with $K$.}
Raking's effective sample size collapses from
$\Neff/N{=}7.1\%$ at $K{=}12$ to $0.12\%$ at $K{=}40$,
where it stabilises ($0.12\%$ also at $K{=}50$).
At the same time, the Shannon entropy of the raking
distribution drops from $9.19$ at $K{=}12$ to
$5.09$ at $K{=}40$---a collapse of $4.1$~nats
corresponding to a factor of $e^{4.1} \approx 60{\times}$
fewer effective profile types.
GibbsPCDSolver maintains $\Neff{=}N$ and entropy
stabilised at $11.51$~nats for $K \ge 30$, yielding
an entropy advantage of $6.4$~nats (${\approx}600{\times}$
in diversity) at $K \ge 40$.
At $K{=}40$, raking simultaneously fails to satisfy
training constraints \emph{and} produces a population
concentrated on ${\approx}120$ archetypes---the worst
of both failure modes.

\begin{figure}[t]
  \centering
  \includegraphics[width=\textwidth]{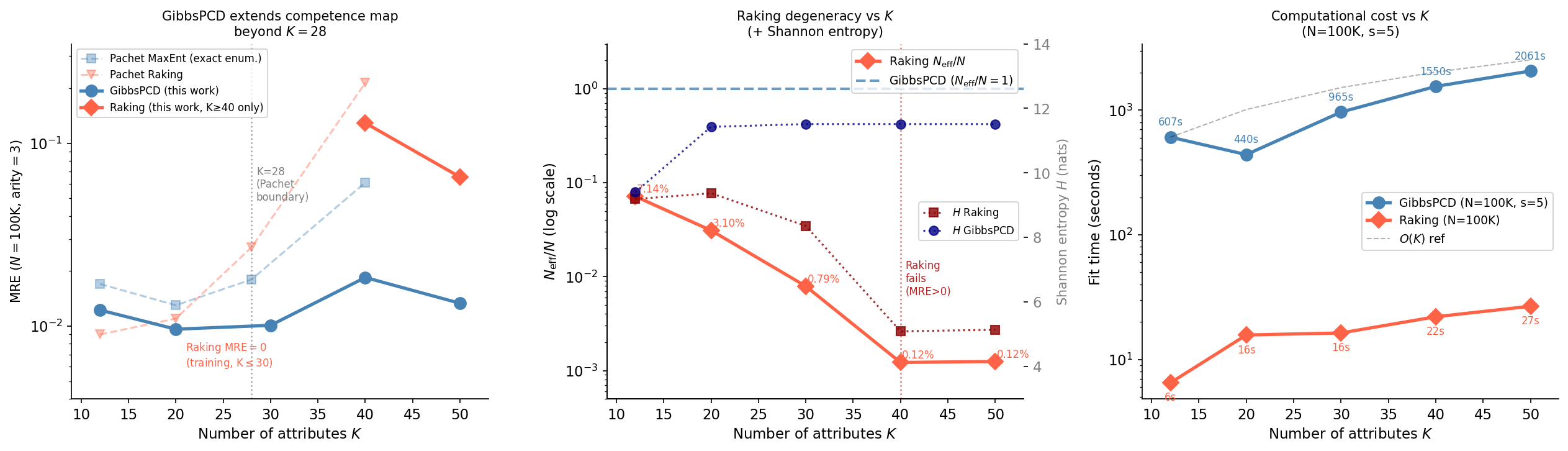}
  \caption{%
    A2 scaling experiments ($N{=}100{,}000$, arity=3).
    \emph{Left}: MRE of GibbsPCDSolver (solid blue) vs.\
    exact MaxEnt of \citet{pachet2026} (dashed blue,
    their Table~2).
    The two curves use different problem instances
    (WuGenerator vs.\ NPORS); the comparison is
    qualitative---both achieve MRE in the $1$--$2\%$
    range---and is not a claim that GibbsPCDSolver
    outperforms exact enumeration.
    At $K{=}30$ only GibbsPCDSolver is available
    ($|\cX| \approx 5{\times}10^{11}$).
    \emph{Centre}: raking $\Neff/N$ (red diamonds,
    log scale, left axis) and Shannon entropy $H$
    (dotted lines, right axis) for both methods.
    \emph{Right}: fit time vs.\ $K$ (log scale).}
  \label{fig:a2_scaling}
\end{figure}

\subsection{Held-Out Generalisation}
\label{sec:exp_heldout}

We next evaluate whether the learned MaxEnt model captures
higher-order structure beyond the training constraints.
Both solvers are trained on unary and binary marginals,
and evaluated on held-out ternary interactions.

Both solvers are trained on the 28 unary and binary
constraints of Syn-ISTAT ($N{=}100{,}000$);
the 3 ternary constraints are withheld and evaluated
after training.

\paragraph{Training MRE.}
GibbsPCDSolver reaches $\MRE_{\text{train}}{=}0.018$
after 600 iterations ($20$~min).
Raking achieves $\MRE_{\text{train}}{=}0$ exactly,
as guaranteed by IPF convergence.
This asymmetry is structural and not indicative of
better generalisation.

\paragraph{Held-out results.}
\Cref{tab:heldout} reports MRE on the three withheld
ternary constraints.

\begin{table}[t]
\centering
\caption{%
  A-ISTAT-2: held-out MRE on ternary constraints.
  Both solvers trained on 28 unary+binary constraints
  ($N{=}100{,}000$). \textbf{Bold}: better result.}
\label{tab:heldout}
\begin{tabular}{lrrrr}
\toprule
Method
  & $\MRE_{\text{train}}$
  & $\MRE_{\text{T1}}$
  & $\MRE_{\text{T2}}$
  & $\MRE_{\text{T3}}$ \\
\midrule
Gibbs PCD & 0.018 & 0.349 & 0.169 & \textbf{0.197} \\
Raking    & 0.000 & \textbf{0.347} & \textbf{0.139} & 0.214 \\
\bottomrule
\multicolumn{5}{l}{%
  T1: edu$\times$emp$\to$income;\quad
  T2: area$\times$car$\to$transport;\quad
  T3: sex$\times$age$\to$employment.}
\end{tabular}
\end{table}

The results reveal a structured pattern that reflects
the topology of the Syn-ISTAT constraint graph.
On T2, raking is substantially better ($0.139$ vs $0.169$):
both conditioning variables of T2 (residential area and
car access) are anchor variables with stable marginals,
and IPF rescaling efficiently propagates their joint
signal.
On T3, GibbsPCDSolver outperforms raking
($\mathbf{0.197}$ vs $0.214$): employment has the highest
degree in the constraint graph (degree~5, connected
to age, sex, education, income, and commute time),
and the MaxEnt exponential family combines these binary
signals coherently through the shared parameter vector
$\blam$.
Raking processes constraints sequentially and
independently, losing the joint information.
On T1 both methods perform similarly.

\subsection{Population Diversity}
\label{sec:exp_diversity}

We now assess the diversity of the generated populations,
which is critical for downstream agent-based simulation.
Even when training accuracy is comparable, methods may differ
substantially in the variety of individuals they produce: a population
with correct marginals but low demographic diversity
produces unrealistic dynamics through excessive
cloning of a few archetypal profiles.
We evaluate diversity with three metrics:
effective sample size $\Neff$ (\cref{eq:neff}),
empirical Shannon entropy
$H = -\sum_{\bx} \hat p(\bx) \log \hat p(\bx)$,
and number of distinct demographic profiles.

\Cref{tab:diversity} reports results for both solvers
on the full Syn-ISTAT problem ($N{=}100{,}000$, 31 constraints).

\begin{table}[t]
\centering
\caption{%
  Population diversity: GibbsPCDSolver vs.\ raking
  ($N{=}100{,}000$, Syn-ISTAT, $K{=}15$).
  $H_{\max} = \log|\cX| = 18.93$ nats.}
\label{tab:diversity}
\begin{tabular}{lrrrr}
\toprule
Method & $\Neff$ & $\Neff/N$ & $H$ (nats) & Unique profiles \\
\midrule
Gibbs PCD & $100{,}000$ & $100\%$  & $11.47$ & $97{,}025$ \\
Raking    & $1{,}152$   & $1.2\%$  & $8.32$  & $25{,}786$ \\
\midrule
Ratio (G/R) & $86.8{\times}$ & --- & $+3.15$ & $3.8{\times}$ \\
\bottomrule
\end{tabular}
\end{table}

The two ratios in the bottom row of \cref{tab:diversity} measure
different things and should not be confused.
Unique profiles counts the support of the pool before weighting:
raking retains $25{,}786$ distinct tuples, but the vast majority
carry negligible weight and will almost never be drawn when
sampling a synthetic population.
$\Neff$, by contrast, captures the effective concentration of
the weight distribution: the $22{\times}$ gap between the two
ratios ($86.8{\times}$ vs.\ $3.8{\times}$) reflects the fact
that raking's $25{,}786$ unique profiles are ``ghost'' individuals
whose combined probability mass is negligible.
The relevant metric for agent-based simulation is $\Neff$:
it bounds the expected number of distinct archetypes in any
downstream sample, regardless of pool support.

The contrast is stark.
Raking converges to an effective population of $1{,}152$
individuals: its weight distribution has Gini
coefficient $0.951$---exceeding the most unequal
empirical income distributions---and the Lorenz curve
(see \cref{fig:diversity}) departs dramatically from
equality.
GibbsPCDSolver maintains uniform pool weights by
construction, achieving $\Neff{=}N$ with $97{,}025$
distinct profiles ($97\%$ of pool size) and an entropy
$3.15$ nats higher than raking.

The practical implication is direct.
For a synthetic municipality of $n$ inhabitants drawn
from the raking distribution, the expected number of
distinct profiles is $\approx \Neff \approx 1{,}200$
regardless of $n$: a city of $400{,}000$ synthetic
Bolognesi would consist almost entirely of repetitions
of $\sim 1{,}200$ archetypes.
GibbsPCDSolver provides the demographic variety that
agent-based urban models require for realistic
emergent behaviour.

This result provides the definitive answer to the
question of why MaxEnt inference via PCD is preferable
to raking even in settings where both achieve comparable
held-out accuracy: raking is algebraically exact on
training constraints but produces \emph{degenerate}
synthetic populations.

\begin{figure}[t]
  \centering
  \includegraphics[width=\textwidth]{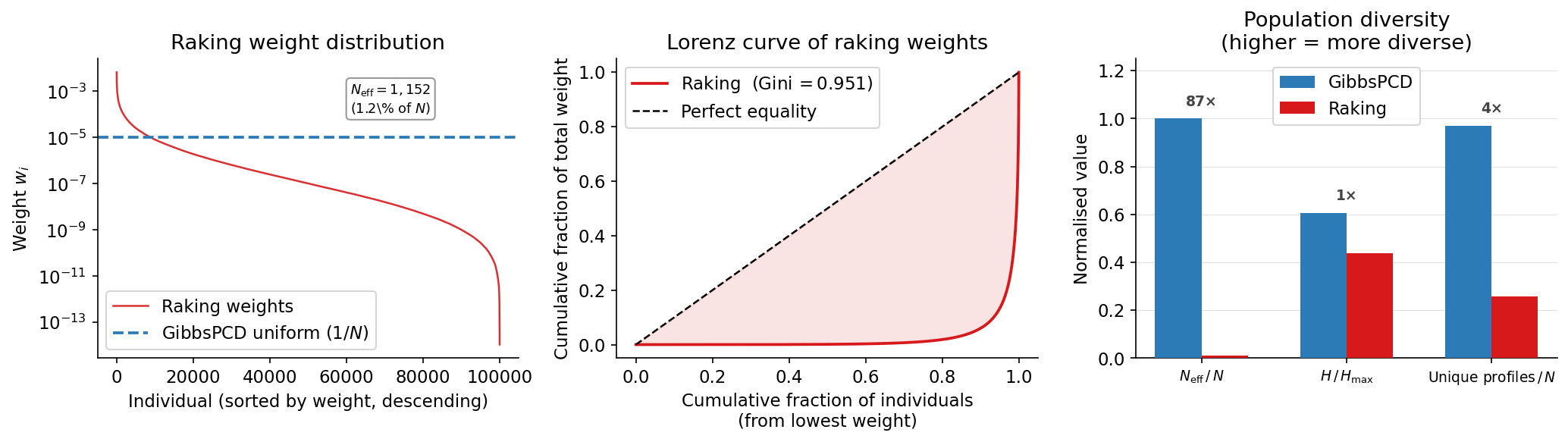}
 \caption{%
  Population diversity: GibbsPCDSolver vs.\ generalised
  raking ($N{=}100{,}000$, Syn-ISTAT, $K{=}15$,
  full training on 31 constraints).
  \emph{Left}: raking weight distribution (log scale);
  the dashed line marks the uniform weight $1/N$ of
  GibbsPCDSolver.
  After convergence, raking concentrates almost all mass
  on $N_{\rm eff}{=}1{,}152$ effective individuals
  ($1.2\%$ of $N$).
  \emph{Centre}: Lorenz curve of raking weights
  (Gini${=}0.951$); the diagonal is perfect equality.
  The shaded area between the two curves quantifies
  the weight concentration.
  \emph{Right}: normalised diversity metrics for both
  solvers; ratios above each pair indicate the
  GibbsPCD advantage.
  GibbsPCDSolver maintains uniform pool weights by
  construction ($N_{\rm eff}{=}N$), achieving an
  $86.8{\times}$ diversity advantage in effective
  sample size and $3.15$~nats higher Shannon entropy
  than raking---essential for realistic agent-based
  dynamics.}

  \label{fig:diversity}
\end{figure}

\subsection{Hyperparameter Sensitivity: Wu Benchmark ($K{=}8$)}
\label{sec:exp_wu}

We analyse the effect of key hyperparameters on controlled
problems where the exact MaxEnt solution is available.

More precisely, we characterise the effect of pool size $N$ and sweep count $s$
on two controlled problems where the exact solution is available:
one with binary constraints only and one with ternary constraints,
both on the same WuGenerator instance ($K{=}8$, $|\cX|{=}6{,}144$).

\paragraph{Binary constraints ($m{=}61$, 25 unary + 36 binary).}
The exact L-BFGS solver reaches $\MRE{=}4{\times}10^{-5}$
in $0.03$~s.
\Cref{tab:a1c} (top block) shows MRE across the full $N \times s$
grid.
The Pareto-optimal configuration is $N{=}50{,}000$, $s{=}1$
($\MRE{=}0.012$, $78$~s): extra sweeps do not help because
gradient variance dominates over mixing bias at large $N$.

\paragraph{Ternary constraints ($m{=}73$, 25 unary + 48 ternary).}
The exact solver reaches $\MRE{=}10^{-4}$ in $0.15$~s.
\Cref{tab:a1c} (bottom block) reveals a different pattern:
at $N{=}50{,}000$ the minimum is at $s{=}10$ ($\MRE{=}0.022$),
and $s{=}1$ gives $\MRE{=}0.025$---only marginally worse at
$11{\times}$ lower cost.
The key finding is that \emph{ternary constraints require larger
$N$ and benefit from more sweeps}: at $N{=}50{,}000$, the ternary
MRE floor ($0.022$) is $1.8{\times}$ higher than the binary
($0.012$), confirming that higher-order interactions impose a
stricter demand on pool diversity.
This directly motivates the choice $s{=}5$ in the Syn-ISTAT
experiments.
\begin{table}[t]
\centering
\caption{%
  A1c: GibbsPCDSolver sensitivity to pool size $N$ and
  sweep count $s$ on the Wu benchmark ($K{=}8$,
  $|\cX|{=}6{,}144$).
  \emph{Top}: binary constraints ($m{=}61$); exact reference
  $\MRE{=}4{\times}10^{-5}$.
  \emph{Bottom}: ternary constraints ($m{=}73$); exact reference
  $\MRE{=}10^{-4}$.
  Time in seconds. \textbf{Bold}: Pareto-optimal entry
  (lowest MRE per unit time) in each block.}
\label{tab:a1c}
\begin{tabular}{r rrr rrr rrr}
\toprule
& \multicolumn{3}{c}{$s=1$}
& \multicolumn{3}{c}{$s=5$}
& \multicolumn{3}{c}{$s=10$} \\
\cmidrule(lr){2-4}\cmidrule(lr){5-7}\cmidrule(lr){8-10}
$N$ & MRE & $t$ && MRE & $t$ && MRE & $t$ & \\
\midrule
\multicolumn{10}{l}{\textit{Binary constraints}} \\
$1{,}000$  & 0.097 &   3 && 0.076 &  11 && 0.062 &  22 & \\
$5{,}000$  & 0.040 &   8 && 0.052 &  38 && 0.039 &  75 & \\
$20{,}000$ & 0.030 &  30 && 0.016 & 141 && 0.023 & 279 & \\
$50{,}000$ & \textbf{0.012} & 78 && 0.015 & 364 && 0.014 & 728 & \\
\midrule
\multicolumn{10}{l}{\textit{Ternary constraints}} \\
$1{,}000$  & 0.137 &   3 && 0.104 &  15 && 0.114 &  29 & \\
$5{,}000$  & 0.065 &  11 && 0.065 &  46 && 0.060 &  91 & \\
$20{,}000$ & 0.035 &  34 && 0.031 & 162 && 0.042 & 323 & \\
$50{,}000$ & 0.025 &  94 && 0.026 & 435 && \textbf{0.022} & 860 & \\
\bottomrule
\end{tabular}
\end{table}
\subsection{Hyperparameter Sensitivity: Syn-ISTAT ($K{=}15$)}
\label{sec:exp_sensitivity}

We finally examine hyperparameter sensitivity on the
Syn-ISTAT benchmark, where exact computation is not
available.
This experiment complements the Wu analysis by showing how
optimal settings depend on constraint structure in a more
realistic scenario: we repeat the sensitivity analysis on Syn-ISTAT,
where $|\cX| \approx 1.7{\times}10^8$ and
the exact solver is unavailable.

\paragraph{Calibration at K=15.}
With $\eta{=}0.05$ and $s{=}1$ (the K=8 defaults),
the solver oscillates between $\MRE \approx 0.08$
and $0.35$ for 300 iterations without converging.
Reducing $\eta$ to $0.01$ and increasing $s$ to $5$
yields monotone descent from $\MRE{=}0.59$ to
$\MRE{=}0.03$ at convergence (535 iterations).
This calibration difference demonstrates that the
optimal sweep count scales with constraint graph
complexity: $s{=}1$ suffices for sparse pairwise
constraints at $K{=}8$, but $s{=}5$ is necessary
for mixed-arity (1+2+3-way) constraints at $K{=}15$.

\paragraph{Pool size results.}
\Cref{tab:a_istat3} summarises MRE over
$N \in \{10{,}000,\, 25{,}000,\, 50{,}000,\,
100{,}000,\, 200{,}000\}$.
The $N{=}50{,}000$ and $N{=}100{,}000$ entries yield
nearly identical MRE ($0.032$), both stopped by the
adaptive rule before reaching the variance floor;
this is an artefact of the stopping window and not a
structural non-monotonicity.
The practical sweet spot is $N{=}25{,}000$:
$\MRE{=}0.050$ in $7$~min, versus
$\MRE{=}0.021$ in $80$~min for $N{=}200{,}000$.

\begin{table}[t]
\centering
\caption{%
  A-ISTAT-3: GibbsPCDSolver pool-size sensitivity on
  Syn-ISTAT ($K{=}15$, 31 constraints, $s{=}5$,
  $\eta{=}0.01$).
  Raking MRE$_{\text{train}}{=}0$ is exact by construction
  and not directly comparable; see \cref{tab:heldout}.}
\label{tab:a_istat3}
\begin{tabular}{lrrrr}
\toprule
Method & $N$ & MRE & Iters & Time (s) \\
\midrule
Gibbs & $10{,}000$  & 0.092 & 323 & 119    \\
Gibbs & $25{,}000$  & 0.050 & 435 & 418    \\
Gibbs & $50{,}000$  & 0.032 & 566 & 1{,}057 \\
Gibbs & $100{,}000$ & 0.032 & 541 & 2{,}074 \\
Gibbs & $200{,}000$ & 0.021 & 590 & 4{,}818 \\
\midrule
Raking & $100{,}000$ & $0.000^\dagger$ & 60 & 29 \\
\bottomrule
\multicolumn{5}{l}{$^\dagger$ Exact on training constraints by construction.}
\end{tabular}
\end{table}

\begin{figure}[t]
  \centering
  \includegraphics[width=\textwidth]{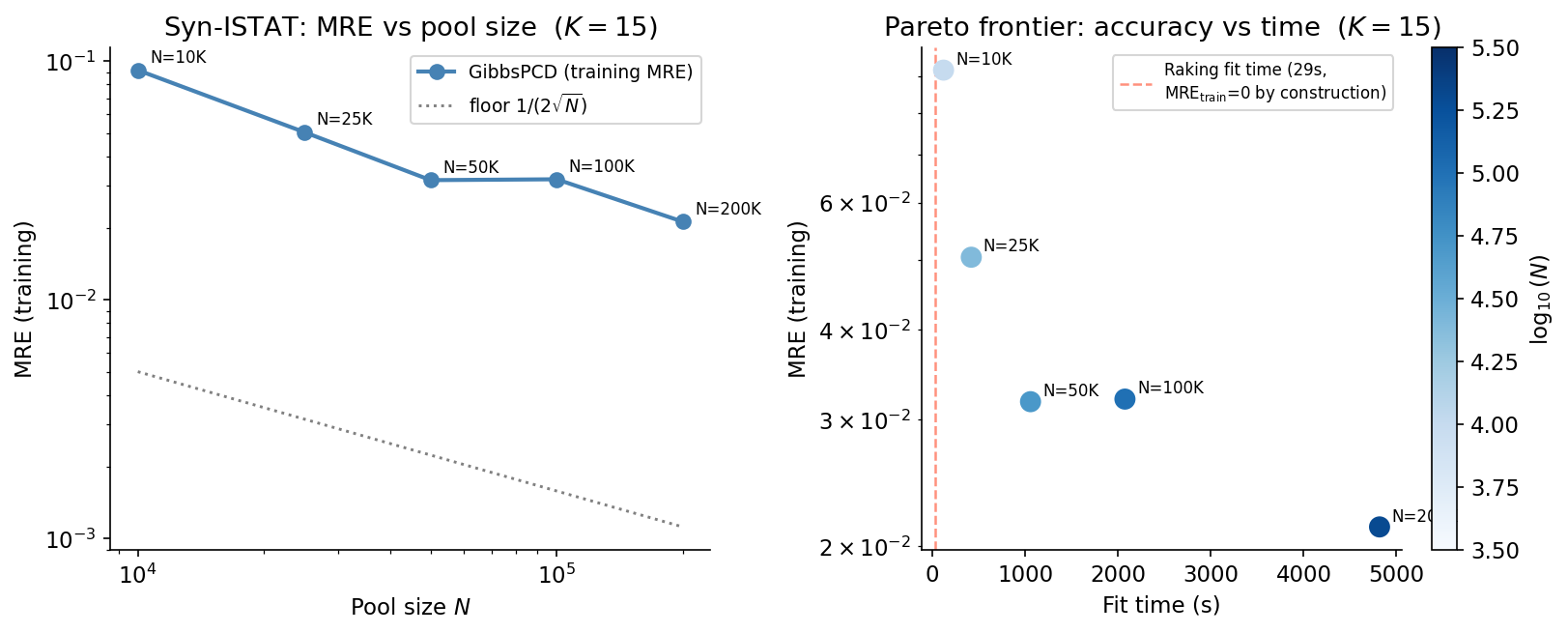}
\caption{%
  A-ISTAT-3 --- pool-size sensitivity on Syn-ISTAT
  ($K{=}15$, $|\cX|{\approx}1.7{\times}10^8$,
  31 constraints, $s{=}5$, $\eta{=}0.01$).
  \emph{Left}: training MRE vs.\ pool size $N$ (log--log);
  the dotted line is the theoretical variance floor
  $1/(2\sqrt{N})$; observed MRE lies above the floor
  at all $N$, confirming that mixing bias dominates
  over gradient variance in this regime.
  \emph{Right}: Pareto frontier of MRE vs.\ fit time;
  colour encodes $\log_{10}(N)$.
  The dashed vertical line marks the fit time of
  generalised raking ($29$~s); raking achieves
  $\MRE_{\text{train}}{=}0$ by algebraic construction
  ($\dagger$) and is not directly comparable to the
  stochastic Gibbs estimate --- see \cref{tab:heldout}
  for a meaningful accuracy comparison on held-out
  constraints.}
  \label{fig:a_istat3_pareto}
\end{figure}

\section{Discussion}
\label{sec:discussion}

\subsection{Extending the Competence Map}

\citet{pachet2026} establish a competence map showing
that MaxEnt overtakes raking when $K \gtrsim 28$ and
ternary interactions are present.
Their map is computed under exact enumeration, which
limits its empirical coverage to $K \le 40$.
GibbsPCDSolver extends this boundary: since the cost
per gradient step is $O(N \cdot K \cdot \bar{J} \cdot d_{\max})$
and is independent of $|\cX|$, the algorithm is in
principle applicable for any $K$.

Our experiments (\cref{sec:exp_scaling}) provide data
points for $K \in \{12, 20, 30, 40, 50\}$, confirming
that GibbsPCDSolver converges reliably across five decades
of $|\cX|$ and that the accuracy--cost trade-off is
controlled by $N$ rather than $K$.
At $K{=}40$, raking fails with $\MRE{=}0.129$, consistent
with the Pachet--Zucker threshold; at $K{=}50$,
$|\cX| \approx 6{\times}10^{18}$ and GibbsPCDSolver
remains the only MaxEnt-based method that can operate,
achieving $\MRE{=}0.013$.

Our experiments (\cref{sec:exp_scaling}) provide data
points for $K \in \{12, 20, 30, 40, 50\}$, confirming
that GibbsPCDSolver converges reliably across five decades
of $|\cX|$ and that the accuracy--cost trade-off is
controlled by $N$ rather than $K$.
At $K{=}40$, raking fails with $\MRE{=}0.129$, consistent
with the Pachet--Zucker threshold; at $K{=}50$,
$|\cX| \approx 6{\times}10^{18}$ and GibbsPCDSolver
remains the only MaxEnt-based method that can operate,
achieving $\MRE{=}0.013$.
This statement is scoped to the aggregate-only regime:
deep generative models can in principle operate at
large $K$, but require individual-level microdata for
training (\cref{sec:bg_deep}), which are often unavailable in many national census.
\subsection{When to Use Which Method}

The experimental results, combined with the diversity
analysis, suggest the following practical decision rule.
For the primary use case of this paper---synthetic
population generation for agent-based urban
simulation---the recommendation is unambiguous and
holds at any scale:

\begin{itemize}
  \item \textbf{Any $K$, downstream agent-based simulation}:
    prefer GibbsPCDSolver over raking regardless of $K$.
    Raking's $\Neff$ degradation is structural and
    grows with $K$: from $\Neff/N{=}7.1\%$ at $K{=}12$
    to $0.12\%$ at $K \ge 40$, making the synthesised
    population unsuitable for behavioural simulation
    regardless of training MRE.
    GibbsPCDSolver maintains $\Neff{=}N$ by construction
    at all scales.

  \item \textbf{$K \le 12$, $|\cX| \lesssim 10^7$}:
    use exact L-BFGS \citep{pachet2026}.
    Enumeration is feasible, the solution is exact,
    and runtime is under $3$~s.

  \item \textbf{$K \in [12, 30]$, mixed or ternary arities}:
    use GibbsPCDSolver with $N{=}50{,}000$--$100{,}000$,
    $s{=}3$--$5$, $\eta{=}0.01$.
    Exact enumeration is borderline or impossible;
    GibbsPCDSolver achieves $\MRE \approx 0.010$--$0.019$
    with runtime of $7$--$16$~minutes at $N{=}100{,}000$.

  \item \textbf{$K \ge 40$}:
    use GibbsPCDSolver with $N \ge 100{,}000$,
    $s{=}5$, and Numba acceleration.
    Raking fails to satisfy training constraints
    in this regime ($\MRE \ge 0.065$ in our experiments)
    and simultaneously degenerates to
    $\Neff \approx 120$ effective individuals.
\end{itemize}

\subsection{Hyperparameter Sensitivity and Calibration}

The main practical finding is that the optimal
learning rate $\eta$ and sweep count $s$ depend on
the complexity of the constraint graph, not on $K$
alone.
The key indicator is the presence of high-arity
constraints connecting high-degree nodes: when such
constraints are present, the stochastic gradient has
higher variance and slower mixing, requiring smaller
$\eta$ and more sweeps per step.
A principled diagnostic is to monitor the trajectory
of $\blam_t$ across iterations: sustained oscillation
of the sequence $(\blam_t)$ signals that $\eta$ is
too large, while monotone but slow descent of MRE
signals that $s$ is too small.
Note that this is a diagnostic on the \emph{dynamics}
of $\blam_t$, not on its distance from $\blam^*$:
as discussed in \cref{sec:validation}, the latter
is not a meaningful convergence metric in the presence
of unary constraints.
In practice, monitoring MRE directly is more reliable
than monitoring $\|\Delta\blam_t\|$ for both
diagnostics.
\subsection{Limitations}

\paragraph{Computational cost.}
Without Numba acceleration, a single run at
$K{=}15$, $N{=}100{,}000$, $s{=}5$ requires
approximately $23$~minutes on a standard CPU.
For production use in urban modelling pipelines,
the Numba-accelerated kernel (available in the
repository) reduces this by a factor of $5$--$15{\times}$.
Further acceleration via batched GPU evaluation
of the Gibbs conditionals is an open engineering
problem.

\paragraph{Approximate convergence.}
GibbsPCDSolver does not guarantee exact satisfaction
of the constraints: the MRE floor is
$\approx 1/(2\sqrt{N})$, and the adaptive stopping
rule may trigger before the true minimum is reached
(as observed for $N{=}50{,}000$ in A-ISTAT-3).
For applications where near-exact satisfaction of
training constraints is the sole objective---and
downstream sampling from the synthesised population
is not required---raking remains preferable.
For synthetic population generation intended for
agent-based simulation, however, this trade-off does
not arise: exact training MRE is irrelevant if the
resulting population degenerates to
${\sim}1{,}200$ effective archetypes regardless of $N$,
as demonstrated in \cref{sec:exp_diversity}.
In that context, an MRE floor of $1/(2\sqrt{N})$
is an acceptable cost for full demographic diversity.

\paragraph{Heterogeneous constraint universes.}
The ISTAT census system publishes marginals with
different reference populations (e.g.\ marital
status for age $\ge 15$, education for age $\ge 25$).
The current formulation assumes a single reference
population for all constraints.
Handling heterogeneous universes requires decomposing
each constraint as a conditional expectation
$\mathbb{E}_p[f_j \mid \bx \in U_j] = \alpha_j$,
where $U_j \subseteq \cX$ is the subpopulation to
which constraint $j$ applies.
This extension preserves the exponential family
structure and the Gibbs conditional derivation of
\cref{prop:gibbs}, but requires a modified gradient
estimator that restricts pool averaging to the
relevant subpopulation at each step.
We leave this extension to future work.

\section{Conclusion}
\label{sec:conclusion}

We have presented two contributions to scalable maximum
entropy population synthesis from aggregate census data.

\textbf{GibbsPCDSolver} replaces the intractable exact
expectation step of \citet{pachet2026} with a Persistent
Contrastive Divergence estimate from a persistent Gibbs
pool---a minimal modification (one line of Algorithm~1)
that removes the $|\cX|$ barrier limiting exact MaxEnt
to $K \lesssim 20$ attributes.
Scaling experiments across $K \in \{12, 20, 30, 40, 50\}$
confirm that GibbsPCDSolver maintains
$\MRE \in [0.010, 0.018]$ while $|\cX|$ grows from
$83{,}000$ to $5.8 \times 10^{18}$---eighteen orders
of magnitude---with runtime scaling as $O(K)$ rather
than $O(|\cX|)$.
At $K{=}40$, raking fails to satisfy training constraints
($\MRE{=}0.129$) and simultaneously degenerates to
$\Neff \approx 120$ effective individuals;
GibbsPCDSolver remains accurate and maintains full
pool diversity at both $K{=}40$ and $K{=}50$.

\textbf{Syn-ISTAT} is a $K{=}15$ Italian demographic
benchmark with analytically exact marginal targets
derived from ISTAT-inspired conditional probability
tables, a training/held-out split that tests
generalisation to third-order interactions, and
$|\cX| \approx 1.7 \times 10^8$ --- placing it
firmly in the non-enumerable regime.
On this benchmark, GibbsPCDSolver converges to
$\MRE{=}0.030$ on training constraints and outperforms
raking on the held-out ternary constraint involving
the highest-degree node in the constraint graph,
where the coherent exponential family parametrisation
exploits multi-hop structure that sequential IPF
rescaling cannot capture.

The most consequential finding, however, is not
about accuracy: GibbsPCDSolver produces populations
with effective sample size $86.8{\times}$ larger and
entropy $3.15$~nats higher than generalised raking.
Raking's weight collapse is structural and grows
monotonically with $K$---from $\Neff/N{=}7.1\%$ at
$K{=}12$ to $0.12\%$ at $K \ge 40$---rendering the
synthesised population unsuitable for agent-based
simulation regardless of training MRE.
For synthetic population generation intended for
urban digital twins and agent-based models,
GibbsPCDSolver is the method of choice at any
scale: not because it achieves lower training error,
but because it produces the demographic diversity
that realistic emergent dynamics require.

\section*{Generative AI Disclosure}
The author used Claude (Anthropic) as an AI assistant during
the preparation of this manuscript, primarily for LaTeX editing,
consistency checking across sections, and refinement of
mathematical notation.
All scientific content, experimental results, and conclusions
are the sole responsibility of the author.

\section*{Acknowledgements}
The author thanks Fran\c{c}ois Pachet (ImagineAllThePeople)
for stimulating discussions on maximum entropy population
synthesis and for making available a preprint of
\citet{pachet2026} prior to publication.


\bibliographystyle{plainnat}
\bibliography{references}

\appendix
\section{Syn-ISTAT: Full Benchmark Specification}
\label{app:synistat}

Syn-ISTAT is an ISTAT-inspired synthetic benchmark: its conditional
probability tables reflect known structural patterns of the Italian
demographic system (gender gap in labour market participation,
age-dependent income profiles, urban/rural transport split) but are
constructed synthetically rather than estimated from official microdata.
This guarantees full ground-truth control and exact analytical computation
of all marginal targets.

\subsection{Attribute Table}
\label{app:attributes}

\begin{table}[h]
\centering
\caption{The 15 categorical attributes of Syn-ISTAT.
  Anchor variables (fixed prior marginals) are marked $\dagger$.
  $|\cX| = \prod_k d_k \approx 1.7 \times 10^8$.}
\label{tab:app_variables}
\small
\begin{tabular}{@{}clcl@{}}
\toprule
\textbf{Index} & \textbf{Attribute} & $d_k$ & \textbf{Domain} \\
\midrule
1  & sex$^\dagger$           & 2 & F, M \\
2  & age$^\dagger$           & 4 & 0--24, 25--49, 50--66, 67+ \\
3  & marital status          & 5 & NeverMarried, Married, Separated, Divorced, Widowed \\
4  & education$^\dagger$     & 4 & LessThanHS, HighSchool, SomeCollege, Bachelor+ \\
5  & employment              & 3 & Employed, Unemployed, NotInLF \\
6  & income                  & 5 & None, Low, Medium, UpperMedium, High \\
7  & household size$^\dagger$& 4 & 1, 2, 3--4, 5+ \\
8  & has children            & 2 & No, Yes \\
9  & residence area$^\dagger$& 3 & Urban, Suburban, Rural \\
10 & car access$^\dagger$    & 3 & NoCar, SharedCar, OwnCar \\
11 & main transport          & 5 & Car, PublicTransport, Walking, Bike, Mixed \\
12 & commute time            & 4 & None, $<$15 min, 15--45 min, 45+ min \\
13 & diet type               & 3 & Omnivore, ReducedMeat, VegetarianLike \\
14 & alcohol use             & 4 & Never, Occasional, Weekly, Frequent \\
15 & physical activity       & 4 & Sedentary, Low, Moderate, High \\
\bottomrule
\end{tabular}
\end{table}

\subsection{Unary Marginals}
\label{app:unary}

Anchor variables (italicised) are drawn from fixed marginals;
non-anchor marginals are derived as implied marginals from the
binary/ternary tables by marginalisation over the conditioning
variable (see \cref{sec:synistat}).
When a variable appears in more than one table, implied marginals
are averaged and renormalised; the maximum pairwise discrepancy
across sources is $0.025$ (employment), which is within acceptable
tolerance for the MaxEnt solver.

\begin{table}[h]
\centering
\caption{Unary marginal distributions for all 15 attributes.
  Anchor distributions are shown in \textit{italics}.}
\label{tab:app_unary}
\small
\begin{tabular}{@{}llr@{}}
\toprule
\textbf{Attribute} & \textbf{Category} & $P$ \\
\midrule
\multirow{2}{*}{\textit{sex}}
  & \textit{F}  & \textit{0.510} \\
  & \textit{M}  & \textit{0.490} \\[3pt]
\multirow{4}{*}{\textit{age}}
  & \textit{0--24}  & \textit{0.240} \\
  & \textit{25--49} & \textit{0.340} \\
  & \textit{50--66} & \textit{0.240} \\
  & \textit{67+}    & \textit{0.180} \\[3pt]
\multirow{5}{*}{marital}
  & NeverMarried & 0.346 \\
  & Married      & 0.426 \\
  & Separated    & 0.038 \\
  & Divorced     & 0.112 \\
  & Widowed      & 0.078 \\[3pt]
\multirow{4}{*}{\textit{education}}
  & \textit{LessThanHS}  & \textit{0.140} \\
  & \textit{HighSchool}  & \textit{0.330} \\
  & \textit{SomeCollege} & \textit{0.280} \\
  & \textit{Bachelor+}   & \textit{0.250} \\[3pt]
\multirow{3}{*}{employment}
  & Employed   & 0.529 \\
  & Unemployed & 0.078 \\
  & NotInLF    & 0.393 \\[3pt]
\multirow{5}{*}{income}
  & None        & 0.127 \\
  & Low         & 0.251 \\
  & Medium      & 0.315 \\
  & UpperMedium & 0.205 \\
  & High        & 0.103 \\[3pt]
\multirow{4}{*}{\textit{household size}}
  & \textit{1}    & \textit{0.280} \\
  & \textit{2}    & \textit{0.340} \\
  & \textit{3--4} & \textit{0.300} \\
  & \textit{5+}   & \textit{0.080} \\[3pt]
\multirow{2}{*}{has children}
  & No  & 0.621 \\
  & Yes & 0.379 \\[3pt]
\multirow{3}{*}{\textit{residence area}}
  & \textit{Urban}    & \textit{0.420} \\
  & \textit{Suburban} & \textit{0.360} \\
  & \textit{Rural}    & \textit{0.220} \\[3pt]
\multirow{3}{*}{\textit{car access}}
  & \textit{NoCar}     & \textit{0.180} \\
  & \textit{SharedCar} & \textit{0.270} \\
  & \textit{OwnCar}    & \textit{0.550} \\[3pt]
\multirow{5}{*}{main transport}
  & Car             & 0.489 \\
  & PublicTransport & 0.211 \\
  & Walking         & 0.103 \\
  & Bike            & 0.065 \\
  & Mixed           & 0.132 \\[3pt]
\multirow{4}{*}{commute time}
  & None   & 0.371 \\
  & $<$15  & 0.192 \\
  & 15--45 & 0.319 \\
  & 45+    & 0.119 \\[3pt]
\multirow{3}{*}{diet type}
  & Omnivore       & 0.759 \\
  & ReducedMeat    & 0.178 \\
  & VegetarianLike & 0.063 \\[3pt]
\multirow{4}{*}{alcohol use}
  & Never      & 0.266 \\
  & Occasional & 0.336 \\
  & Weekly     & 0.280 \\
  & Frequent   & 0.118 \\[3pt]
\multirow{4}{*}{physical activity}
  & Sedentary & 0.236 \\
  & Low       & 0.290 \\
  & Moderate  & 0.298 \\
  & High      & 0.175 \\
\bottomrule
\end{tabular}
\end{table}

\subsection{Binary Contingency Tables}
\label{app:binary}

All 13 binary tables specify $P(B \mid A)$; each row sums to~1.
The conditioning variable $A$ is always an anchor or a variable
with a well-established prior; $B$ is the conditioned attribute.

\begin{table}[h]
\centering
\caption{Index of the 13 binary contingency tables.}
\label{tab:app_binary_index}
\small
\begin{tabular}{@{}llll@{}}
\toprule
\textbf{\#} & \textbf{Table} & \textbf{Conditioning} & \textbf{Conditioned} \\
\midrule
B1  & \texttt{age\_marital}          & age            & marital status \\
B2  & \texttt{age\_employment}       & age            & employment \\
B3  & \texttt{education\_employment} & education      & employment \\
B4  & \texttt{employment\_income}    & employment     & income \\
B5  & \texttt{household\_children}   & household size & has children \\
B6  & \texttt{area\_transport}       & residence area & main transport \\
B7  & \texttt{car\_transport}        & car access     & main transport \\
B8  & \texttt{age\_alcohol}          & age            & alcohol use \\
B9  & \texttt{sex\_employment}       & sex            & employment \\
B10 & \texttt{sex\_income}           & sex            & income \\
B11 & \texttt{employment\_commute}   & employment     & commute time \\
B12 & \texttt{age\_diet}             & age            & diet type \\
B13 & \texttt{age\_activity}         & age            & physical activity \\
\bottomrule
\end{tabular}
\end{table}

\paragraph{B1.} $P(\text{marital} \mid \text{age})$

\begin{tabular}{@{}lrrrrr@{}}
\toprule
Age & NeverMarr. & Married & Separated & Divorced & Widowed \\
\midrule
0--24  & 0.880 & 0.100 & 0.005 & 0.010 & 0.005 \\
25--49 & 0.270 & 0.550 & 0.050 & 0.110 & 0.020 \\
50--66 & 0.120 & 0.580 & 0.060 & 0.180 & 0.060 \\
67+    & 0.080 & 0.420 & 0.030 & 0.160 & 0.310 \\
\bottomrule
\end{tabular}

\medskip
\paragraph{B2.} $P(\text{employment} \mid \text{age})$\quad
\textbf{B3.} $P(\text{employment} \mid \text{education})$

\begin{tabular}{@{}lrrr@{}}
\toprule
Age & Employed & Unemployed & NotInLF \\
\midrule
0--24  & 0.380 & 0.120 & 0.500 \\
25--49 & 0.740 & 0.080 & 0.180 \\
50--66 & 0.630 & 0.070 & 0.300 \\
67+    & 0.080 & 0.020 & 0.900 \\
\bottomrule
\end{tabular}
\quad
\begin{tabular}{@{}lrrr@{}}
\toprule
Education & Employed & Unemployed & NotInLF \\
\midrule
LessThanHS  & 0.390 & 0.110 & 0.500 \\
HighSchool  & 0.560 & 0.090 & 0.350 \\
SomeCollege & 0.610 & 0.080 & 0.310 \\
Bachelor+   & 0.710 & 0.050 & 0.240 \\
\bottomrule
\end{tabular}

\medskip
\paragraph{B4.} $P(\text{income} \mid \text{employment})$

\begin{tabular}{@{}lrrrrr@{}}
\toprule
Employment & None & Low & Medium & UpperMedium & High \\
\midrule
Employed   & 0.020 & 0.180 & 0.380 & 0.280 & 0.140 \\
Unemployed & 0.240 & 0.460 & 0.230 & 0.060 & 0.010 \\
NotInLF    & 0.250 & 0.310 & 0.250 & 0.130 & 0.060 \\
\bottomrule
\end{tabular}

\medskip
\paragraph{B5.} $P(\text{has children} \mid \text{household size})$\quad
\textbf{B8.} $P(\text{alcohol use} \mid \text{age})$

\begin{tabular}{@{}lrr@{}}
\toprule
Hh.\ size & No & Yes \\
\midrule
1    & 0.980 & 0.020 \\
2    & 0.730 & 0.270 \\
3--4 & 0.290 & 0.710 \\
5+   & 0.140 & 0.860 \\
\bottomrule
\end{tabular}
\quad
\begin{tabular}{@{}lrrrr@{}}
\toprule
Age & Never & Occasional & Weekly & Frequent \\
\midrule
0--24  & 0.370 & 0.360 & 0.200 & 0.070 \\
25--49 & 0.180 & 0.310 & 0.350 & 0.160 \\
50--66 & 0.200 & 0.330 & 0.320 & 0.150 \\
67+    & 0.380 & 0.360 & 0.200 & 0.060 \\
\bottomrule
\end{tabular}

\medskip
\paragraph{B6.} $P(\text{main transport} \mid \text{residence area})$\quad
\textbf{B7.} $P(\text{main transport} \mid \text{car access})$

\begin{tabular}{@{}lrrrrr@{}}
\toprule
Area & Car & PublicTr. & Walking & Bike & Mixed \\
\midrule
Urban    & 0.290 & 0.340 & 0.140 & 0.090 & 0.140 \\
Suburban & 0.580 & 0.160 & 0.080 & 0.050 & 0.130 \\
Rural    & 0.720 & 0.050 & 0.070 & 0.040 & 0.120 \\
\bottomrule
\end{tabular}
\quad
\begin{tabular}{@{}lrrrrr@{}}
\toprule
Car acc. & Car & PublicTr. & Walking & Bike & Mixed \\
\midrule
NoCar     & 0.020 & 0.420 & 0.240 & 0.120 & 0.200 \\
SharedCar & 0.410 & 0.200 & 0.080 & 0.050 & 0.260 \\
OwnCar    & 0.680 & 0.100 & 0.050 & 0.040 & 0.130 \\
\bottomrule
\end{tabular}

\medskip
\paragraph{B9.} $P(\text{employment} \mid \text{sex})$\quad
\textbf{B10.} $P(\text{income} \mid \text{sex})$

\begin{tabular}{@{}lrrr@{}}
\toprule
Sex & Employed & Unemployed & NotInLF \\
\midrule
F & 0.480 & 0.090 & 0.430 \\
M & 0.620 & 0.070 & 0.310 \\
\bottomrule
\end{tabular}
\quad
\begin{tabular}{@{}lrrrrr@{}}
\toprule
Sex & None & Low & Medium & UpperMedium & High \\
\midrule
F & 0.170 & 0.280 & 0.310 & 0.160 & 0.080 \\
M & 0.070 & 0.210 & 0.320 & 0.260 & 0.140 \\
\bottomrule
\end{tabular}

\medskip
\paragraph{B11.} $P(\text{commute time} \mid \text{employment})$

\begin{tabular}{@{}lrrrr@{}}
\toprule
Employment & None & $<$15 min & 15--45 min & 45+ min \\
\midrule
Employed   & 0.050 & 0.220 & 0.520 & 0.210 \\
Unemployed & 0.600 & 0.180 & 0.170 & 0.050 \\
NotInLF    & 0.720 & 0.160 & 0.100 & 0.020 \\
\bottomrule
\end{tabular}

\medskip
\paragraph{B12.} $P(\text{diet type} \mid \text{age})$\quad
\textbf{B13.} $P(\text{physical activity} \mid \text{age})$

\begin{tabular}{@{}lrrr@{}}
\toprule
Age & Omnivore & ReducedMeat & VegetarianLike \\
\midrule
0--24  & 0.680 & 0.220 & 0.100 \\
25--49 & 0.720 & 0.210 & 0.070 \\
50--66 & 0.810 & 0.150 & 0.040 \\
67+    & 0.870 & 0.100 & 0.030 \\
\bottomrule
\end{tabular}
\quad
\begin{tabular}{@{}lrrrr@{}}
\toprule
Age & Sedentary & Low & Moderate & High \\
\midrule
0--24  & 0.120 & 0.210 & 0.360 & 0.310 \\
25--49 & 0.200 & 0.270 & 0.330 & 0.200 \\
50--66 & 0.280 & 0.340 & 0.280 & 0.100 \\
67+    & 0.400 & 0.370 & 0.180 & 0.050 \\
\bottomrule
\end{tabular}

\subsection{Ternary Contingency Tables}
\label{app:ternary}

Three ternary constraints capture higher-order interactions that
cannot be factored into pairwise terms without significant
distortion.
Each ternary table $P(C \mid A, B)$ nominally supersedes the
corresponding binary tables $P(C \mid A)$ and $P(C \mid B)$,
but both are retained in the constraint set to maintain a richer
constraint graph topology.

\paragraph{T1.} $P(\text{income} \mid \text{education},\, \text{employment})$

\noindent Captures the joint effect of human capital and labour
market status on income: the education premium on earnings is
visible only when conditioning on employment status jointly.

\begin{table}[h]
\centering
\small
\begin{tabular}{@{}llrrrrr@{}}
\toprule
Education & Employment & None & Low & Medium & UpperMedium & High \\
\midrule
\multirow{3}{*}{LessThanHS}
  & Employed   & 0.030 & 0.320 & 0.400 & 0.180 & 0.070 \\
  & Unemployed & 0.300 & 0.480 & 0.180 & 0.030 & 0.010 \\
  & NotInLF    & 0.300 & 0.380 & 0.220 & 0.080 & 0.020 \\[3pt]
\multirow{3}{*}{HighSchool}
  & Employed   & 0.020 & 0.210 & 0.430 & 0.240 & 0.100 \\
  & Unemployed & 0.240 & 0.470 & 0.220 & 0.060 & 0.010 \\
  & NotInLF    & 0.240 & 0.330 & 0.260 & 0.120 & 0.050 \\[3pt]
\multirow{3}{*}{SomeCollege}
  & Employed   & 0.010 & 0.160 & 0.390 & 0.290 & 0.150 \\
  & Unemployed & 0.220 & 0.440 & 0.250 & 0.070 & 0.020 \\
  & NotInLF    & 0.220 & 0.290 & 0.270 & 0.150 & 0.070 \\[3pt]
\multirow{3}{*}{Bachelor+}
  & Employed   & 0.010 & 0.080 & 0.280 & 0.370 & 0.260 \\
  & Unemployed & 0.180 & 0.390 & 0.280 & 0.110 & 0.040 \\
  & NotInLF    & 0.180 & 0.210 & 0.250 & 0.200 & 0.160 \\
\bottomrule
\end{tabular}
\end{table}

\paragraph{T2.} $P(\text{main transport} \mid \text{residence area},\, \text{car access})$

\noindent Captures the interaction between residential context and
vehicle availability in modal choice: neither binary table B6 nor
B7 alone can reproduce the Rural$\times$OwnCar dominance of
private car ($0.840$).

\begin{table}[h]
\centering
\small
\begin{tabular}{@{}llrrrrr@{}}
\toprule
Area & Car access & Car & PublicTr. & Walking & Bike & Mixed \\
\midrule
\multirow{3}{*}{Urban}
  & NoCar     & 0.010 & 0.480 & 0.250 & 0.100 & 0.160 \\
  & SharedCar & 0.330 & 0.240 & 0.100 & 0.070 & 0.260 \\
  & OwnCar    & 0.560 & 0.150 & 0.080 & 0.050 & 0.160 \\[3pt]
\multirow{3}{*}{Suburban}
  & NoCar     & 0.030 & 0.430 & 0.220 & 0.100 & 0.220 \\
  & SharedCar & 0.460 & 0.180 & 0.070 & 0.040 & 0.250 \\
  & OwnCar    & 0.720 & 0.080 & 0.040 & 0.030 & 0.130 \\[3pt]
\multirow{3}{*}{Rural}
  & NoCar     & 0.040 & 0.220 & 0.300 & 0.100 & 0.340 \\
  & SharedCar & 0.580 & 0.080 & 0.070 & 0.030 & 0.240 \\
  & OwnCar    & 0.840 & 0.030 & 0.030 & 0.020 & 0.080 \\
\bottomrule
\end{tabular}
\end{table}

\paragraph{T3.} $P(\text{employment} \mid \text{sex},\, \text{age})$

\noindent Captures the age-dependent Italian gender gap in labour
market participation: the gap is moderate in the youngest cohort
(high \texttt{NotInLF} rates for both sexes due to education),
maximal in the prime working-age group 25--49 ($87\%$ male
vs.\ $61\%$ female employment rate), and effectively disappears
post-retirement.
Neither binary table B2 nor B9 can represent this interaction.
\emph{Consistency check}: marginalising over sex
($P(\text{F}){=}0.51$, $P(\text{M}){=}0.49$) recovers
$P(\text{employed} \mid \text{age})$ to within $\pm 0.01$ of
B2; marginalising over age recovers $P(\text{employed} \mid
\text{sex})$ to within $\pm 0.06$ of B9.

\begin{table}[h]
\centering
\small
\begin{tabular}{@{}llrrr@{}}
\toprule
Sex & Age & Employed & Unemployed & NotInLF \\
\midrule
\multirow{4}{*}{F}
  & 0--24  & 0.330 & 0.140 & 0.530 \\
  & 25--49 & 0.610 & 0.100 & 0.290 \\
  & 50--66 & 0.520 & 0.070 & 0.410 \\
  & 67+    & 0.050 & 0.010 & 0.940 \\[3pt]
\multirow{4}{*}{M}
  & 0--24  & 0.430 & 0.100 & 0.470 \\
  & 25--49 & 0.870 & 0.060 & 0.070 \\
  & 50--66 & 0.740 & 0.070 & 0.190 \\
  & 67+    & 0.110 & 0.020 & 0.870 \\
\bottomrule
\end{tabular}
\end{table}

\subsection{Constraint Graph}
\label{app:graph}

With the 13 binary and 3 ternary constraints, all 15 variables
are connected; no variable is isolated from the graph.
High-degree nodes are \textbf{age} (degree~7: marital, employment,
diet, alcohol, activity, plus ternaries T1 and T3) and
\textbf{employment} (degree~5: age, education, income, commute,
sex, plus T1 and T3).
Degree-1 nodes (connected to exactly one other variable) are:
marital (age only), has children (household size only),
commute time (employment only), diet type (age only),
physical activity (age only), and alcohol use (age only).
This topology motivates the held-out experiment in
\cref{sec:exp_heldout}: employment's high degree makes T3
($\text{sex} \times \text{age} \to \text{employment}$)
the most demanding held-out constraint, and the one on which
GibbsPCDSolver outperforms raking.

\end{document}